\documentclass[11pt]{article}

\usepackage[margin=1in]{geometry}

\usepackage[utf8]{inputenc}
\usepackage[T1]{fontenc}
\usepackage{hyperref}
\usepackage{url}
\usepackage{booktabs}
\usepackage{amsfonts}
\usepackage{nicefrac}
\usepackage{amsmath}
\usepackage{amssymb}
\usepackage{amsthm}
\usepackage{graphicx}
\usepackage{algorithm}
\usepackage{algorithmic}
\usepackage{multirow}
\usepackage{xcolor}
\usepackage{subcaption}
\usepackage{tcolorbox}
\usepackage{enumerate}

\newcommand{\sign}{\text{sign}}

\newcommand{\mHC}{\textit{m}HC}

\newtheorem{theorem}{Theorem}
\newtheorem{proposition}{Proposition}

\newtheorem{remark}{Remark}
\newtheorem{corollary}{Corollary}
\newtheorem{conjecture}{Conjecture}

\title{Differentiable Logic Synthesis: Spectral Coefficient Selection via Sinkhorn-Constrained Composition}

\author{
  \textbf{Gorgi Pavlov, Ph.D.} \\
  Lehigh University \& Johnson and Johnson \\
  \texttt{gorgipavlov@gmail.com} \\
}

\begin{document}

\maketitle

\begin{abstract}
Learning precise Boolean logic via gradient descent remains challenging: neural networks typically converge to ``fuzzy'' approximations that degrade under quantization. We introduce \textbf{Hierarchical Spectral Composition}, a \textit{transparent-by-design} architecture that selects spectral coefficients from an interpretable Boolean Fourier basis and composes them via Sinkhorn-constrained routing with column-sign modulation. Each Fourier coefficient $\hat{f}(S)$ has explicit semantic meaning---the correlation of $f$ with parity of variables $S$---making learned representations intrinsically explainable. Our approach draws on Manifold-Constrained Hyper-Connections (\mHC) \cite{xie2024mhc}, adapting Birkhoff polytope projection from LLM training stability to logic synthesis, with column-sign modulation enabling Boolean negation.

We validate across five phases: (1--3) For $n=2,3,4$, we achieve 100\% accuracy on canonical Boolean operations via gradient descent ($n=2$), exhaustive enumeration ($n=3$), and spectral synthesis with MCMC refinement ($n=4$), all converging to ternary masks $\{-1,0,+1\}$ enabling single-cycle GPU inference at 10,959 MOps/s. (4) Exact Walsh-Hadamard transforms scale to $n=28$ (268M coefficients, 1.64B coeffs/sec). (5) \textbf{Oracle learning experiments} ($n=16$) compare five coefficient estimation methods (Monte Carlo, Goldreich-Levin, spectral filtering) on parity, majority, and comparator functions, revealing that \textit{symbolic structure beats black-box learning}: exploiting symmetry constraints boosts majority accuracy from 72\% (MC) to 86\% (GL+Symmetry), a +38\% gain ($p < 0.001$), while Birkhoff projection aids routing but not coefficient denoising. These findings establish that \textbf{domain knowledge integration}---encoding known function properties as hard spectral constraints---produces both higher accuracy and greater interpretability than generic learning algorithms, aligning with the vision of explainable neurosymbolic AI.
\end{abstract}

\section{Introduction}

The integration of symbolic reasoning with gradient-based learning remains a fundamental challenge in artificial intelligence. Neural networks excel at continuous pattern recognition but struggle with tasks requiring \textit{exact} discrete logic: they approximate Boolean functions with soft decision boundaries that degrade under distributional shift, adversarial perturbation, or---critically for deployment---quantization. Existing neuro-symbolic approaches, such as Neural Arithmetic Logic Units (NALU) \cite{trask2018neural}, Neural Logic Machines \cite{dong2019neural}, Logical Neural Networks \cite{riegel2020logical}, or Logic Tensor Networks \cite{serafini2016learning}, either require extensive supervision, rely on relaxed continuous operators that resist discretization, or fail to generalize beyond their training distribution.

We propose a different paradigm: \textbf{Spectral Selection and Composition}. Rather than learning logic gates from random dense initializations, we ground our architecture in the Fourier analysis of Boolean functions \cite{odonnell2014analysis}.

\paragraph{Boolean Fourier Analysis.} Any function $f: \{-1, +1\}^n \to \mathbb{R}$ has a unique \textit{Fourier expansion}:
\begin{equation}
\label{eq:fourier}
    f(x) = \sum_{S \subseteq [n]} \hat{f}(S) \chi_S(x), \quad \text{where} \quad \chi_S(x) = \prod_{i \in S} x_i
\end{equation}
and $\hat{f}(S) = \mathbb{E}_{x}[f(x) \chi_S(x)]$ are the exact Fourier coefficients. When $f$ maps to $\{-1, +1\}$, these coefficients are determined uniquely by the truth table.

\paragraph{Polynomial Threshold Representations.} Our architecture learns a different object: a \textbf{ternary polynomial threshold function} (PTF):
\begin{equation}
\label{eq:ptf}
    \hat{y}(x) = \sign\left(\sum_{S \subseteq [n]} w_S \chi_S(x)\right), \quad w_S \in \{-1, 0, +1\}
\end{equation}
The key insight is that for many Boolean functions, there exist \textit{sparse} ternary weight vectors $w$ such that $\sign(w^\top \phi(x)) = f(x)$ for all $x$---even though the exact Fourier coefficients $\hat{f}(S)$ may be non-integer. The architecture's task is not to \textit{invent} the Walsh-Hadamard basis---that is classical mathematics \cite{linial1993constant}---but to \textit{discover sparse ternary PTF representations} via gradient descent and \textit{compose} primitives into complex operations via learned routing.

\paragraph{Connection to \mHC.} Our work builds directly on insights from Manifold-Constrained Hyper-Connections (\mHC) \cite{xie2024mhc}, which identified a fundamental problem in modern deep learning: unconstrained routing matrices compromise the identity mapping property \cite{he2016identity}, causing signal explosion/vanishing across layers. Their solution---projecting routing matrices onto the Birkhoff polytope via Sinkhorn-Knopp iterations \cite{sinkhorn1967concerning}---restores stability by ensuring doubly stochastic constraints. We adapt this framework to a different domain (logic synthesis vs. LLM training) and extend it with \textbf{column-sign modulation} to enable Boolean negation, which pure doubly stochastic matrices cannot express.

\paragraph{Why Start with $n=2$?} We deliberately use the two-variable case as an \textbf{architecture validation testbed}. With only 4 input combinations and 16 possible functions, a lookup table (LUT) trivially solves the task with zero gradient steps. However, a LUT cannot generalize, cannot be learned end-to-end as part of a larger differentiable system, and cannot be composed hierarchically. Our contribution is demonstrating that Sinkhorn-constrained spectral composition \textit{finds the exact discrete solution} via continuous optimization---then validating that this mechanism scales to $n=3$ and $n=4$.

\paragraph{The Hardware Motivation.} Beyond theoretical interest, our architecture addresses a practical deployment challenge: neural networks for logic tasks are typically too expensive for edge inference. By converging to \textbf{ternary masks} ($\{-1, 0, +1\}$) with \textbf{hard $k=1$ routing}, our learned models compile directly to combinational logic blocks requiring no floating-point arithmetic, no multipliers, and minimal memory. We demonstrate \textbf{10,959 MOps/s} throughput on GPU---approaching the efficiency of hand-coded RTL.

\paragraph{Explainability and Transparent-by-Design AI.} A central motivation for spectral methods is \textbf{interpretability}: each Fourier coefficient $\hat{f}(S)$ has explicit semantic meaning as the correlation between $f$ and the parity character $\chi_S$. Unlike deep network activations that resist human interpretation, spectral coefficients are \textit{ground truth features} defined by classical Boolean function analysis \cite{odonnell2014analysis}. Our architecture's transparency extends beyond coefficient interpretability: (1) ternary weights $\{-1,0,+1\}$ are human-readable (no hidden floating-point parameters), (2) Sinkhorn routing matrices are column-stochastic (each child is a convex combination of parents, preserving interpretability through composition), and (3) learned logic can be \textit{verified} via formal methods or exhaustive testing (unlike black-box neural approximations). The oracle learning experiments (Phase 5) further demonstrate that \textbf{symbolic knowledge integration}---encoding known structural properties (symmetry, degree bounds) as hard constraints on the Fourier basis---yields both higher accuracy and greater explainability than generic learning algorithms. This aligns with the vision of explainable neurosymbolic AI: combining continuous optimization with discrete symbolic structures to produce models whose reasoning processes are open to inspection and human-level understanding.

\paragraph{Contributions.} Our contributions span architecture design, theoretical analysis, and empirical validation across five phases:

\begin{enumerate}
    \item \textbf{Adaptation of \mHC{} to Logic Synthesis:} We demonstrate that the Birkhoff polytope projection, originally developed for LLM training stability \cite{xie2024mhc}, enables stable optimization in Boolean logic composition.
    
    \item \textbf{Column-Sign Modulation for Negation:} We extend doubly stochastic routing with learned sign parameters $s \in \{-1, +1\}^n$, solving the expressivity gap that prevents standard \mHC{} from representing Boolean negations (NAND, NOR, XNOR).
    
    \item \textbf{Sign-Only Diagnostic Validation:} We prove the column-sign mechanism works independently of Sinkhorn optimization by achieving 100\% accuracy with fixed identity routing and learned signs only.
    
    \item \textbf{Multi-Scale Validation:}
    \begin{itemize}
        \item \textbf{$n=2$:} 100\% accuracy on all 16 operations, zero routing drift, zero-loss quantization (10/10 seeds)
        \item \textbf{$n=3$:} 100\% accuracy on 10 operations including majority and parity, 39\% sparsity (5/5 seeds)
        \item \textbf{$n=4$:} 100\% accuracy on 10 operations via spectral synthesis with MCMC refinement, 36\% sparsity (5/5 seeds)
        \item \textbf{Scalability:} Exact FWHT at 1.64B coeffs/sec ($n \leq 28$, 268M coefficients); hierarchical composition for 64-bit adders
    \end{itemize}

    \item \textbf{Oracle Learning and Symbolic Knowledge Integration:} We implement and compare five coefficient estimation methods (MC, Goldreich-Levin, spectral filtering) on three function families at $n=16$, demonstrating that \textit{symbolic structure beats black-box learning}: exploiting symmetry constraints boosts majority accuracy from 72\% to 86\% (+38\%, $p < 0.001$), while establishing that Birkhoff projection aids routing but not coefficient denoising.

    \item \textbf{Hardware-Efficient Compilation:} We achieve 10,959 MOps/s on GPU with ternary masks, demonstrating viability for single-cycle combinational logic inference.
\end{enumerate}

\section{Related Work}

\subsection{Manifold-Constrained Routing: The \mHC{} Foundation}

Our work is most directly related to Manifold-Constrained Hyper-Connections (\mHC) \cite{xie2024mhc}, which provides the theoretical and empirical foundation for stable Sinkhorn-constrained routing. \mHC{} identified that unconstrained Hyper-Connections \cite{zhu2024hc} suffer from signal explosion: the composite mapping $\prod_{i=1}^{L-l} H^{\text{res}}_{L-i}$ across layers fails to preserve signal magnitude, with empirical measurements showing \textbf{Amax Gain Magnitude peaks of 3000$\times$} in 27B models (compared to a theoretical target of 1.0$\times$). Their solution projects $H^{\text{res}}_l$ onto the Birkhoff polytope via Sinkhorn-Knopp \cite{sinkhorn1967concerning, cuturi2013sinkhorn}:
\begin{equation}
    \mathcal{P}_{\mathcal{M}^{\text{res}}}(H^{\text{res}}_l) = \{H \in \mathbb{R}^{n \times n} \mid H\mathbf{1}_n = \mathbf{1}_n, \mathbf{1}_n^\top H = \mathbf{1}_n^\top, H \geq 0\}
\end{equation}
This ensures: (1) norm preservation ($\|H\|_2 \leq 1$), (2) compositional closure under matrix multiplication, and (3) geometric interpretation as convex combinations of permutations---following the Birkhoff-von Neumann theorem.

\paragraph{Our Extension.} While \mHC{} focuses on LLM training stability at scale (3B--27B parameters), we adapt the framework to a fundamentally different problem: learning discrete Boolean logic. This requires an extension not present in \mHC: \textbf{column-sign modulation}. Doubly stochastic matrices are nonnegative, so they can only produce convex combinations of inputs. Boolean negation (e.g., NAND $= -$AND) lies \textit{outside} this convex hull. Our factorization $R = P \cdot s[\text{None}, :]$ preserves the stability benefits of Sinkhorn projection while adding 1-bit polarity control per output channel.

\paragraph{Related Optimal Transport Approaches.} Gumbel-Sinkhorn networks \cite{mena2018learning} use differentiable Sinkhorn iterations for learning permutations, while Sinkformers \cite{sander2022sinkformers} apply doubly stochastic attention in transformers. Our work differs in targeting discrete Boolean outputs rather than soft permutations.

\subsection{Recent Differentiable Logic Gate Networks (2024-2025)}

Three concurrent works directly compete with our spectral approach:

\paragraph{WARP-LUTs} (arXiv 2510.15655, 2025) represents our closest competitor, using Walsh-Hadamard spectral representation for differentiable LUT training via probabilistic relaxation. \textbf{Critical distinction}: WARP-LUTs employs soft probabilistic gates while our approach uses deterministic Sinkhorn-constrained routing with column-sign modulation. Our Proposition 3 proves doubly stochastic matrices cannot represent Boolean negation---a fundamental expressivity gap WARP-LUTs inherits from probabilistic methods.

\paragraph{Mind the Gap} (Yousefi et al., arXiv 2506.07500, June 2025) tackles discrete logic learning via Gumbel-Softmax with straight-through estimators, achieving 4.5$\times$ faster training and 98\% reduction in discretization gap. Our approach achieves exactness through \textit{structured geometric constraints} (Birkhoff polytope) rather than stochastic sampling, offering provable zero-loss quantization versus their empirical gap reduction.

\paragraph{Convolutional DLGNs} (Petersen et al., NeurIPS 2024 Oral) scales differentiable logic gate networks to 86.29\% on CIFAR-10 using 61M logic gates via softmax relaxation over 16 gate types. Unlike their gate enumeration, we operate in the \textit{mathematically dual Fourier space}---selecting spectral coefficients rather than enumerating gates. This duality offers complementary strengths: CDLGNs scale via convolution, we scale via spectral hierarchy.

\subsection{Spectral Learning of Boolean Functions}

The Fourier analysis of Boolean functions is foundational to computational learning theory \cite{linial1993constant, kushilevitz1993learning}. O'Donnell's textbook \cite{odonnell2014analysis} establishes that functions with low circuit complexity have concentrated Fourier spectra.

\paragraph{Recent Theoretical Advances.} \textbf{Hidden Progress in Deep Learning} (Barak et al., NeurIPS 2022) shows neural networks learn k-sparse parities via a \textit{Fourier gap mechanism}: SGD gradually amplifies sparse spectral solutions, with progress invisible until phase transition. This mechanism relates to our explicit coefficient selection---we directly identify what SGD implicitly amplifies. \textbf{Hardness of Learning Fixed Parities} (Shoshani \& Shamir, arXiv 2501.00817, January 2025) proves gradient descent on one-hidden-layer ReLU networks fails for any fixed parity, establishing new bounds on Fourier coefficient decay. Our explicit spectral coefficient selection potentially circumvents this hardness.

\paragraph{Spectral Bias and Regularization.} \textbf{A Scalable Walsh-Hadamard Regularizer} (Gorji et al., UAI 2023) proves neural networks are biased toward low-degree Walsh-Hadamard coefficients, hurting Boolean function generalization. Their solution: functional regularization to learn higher-degree coefficients. We invert this paradigm---directly selecting coefficients from a frozen spectral basis rather than regularizing against bias. The \textbf{Frequency Principle} (Xu et al., updated 2024) establishes DNNs fit target functions from low to high frequencies, explaining why gradient descent struggles with high-frequency Boolean functions. This foundational result justifies why domain-specific spectral structure outperforms black-box learning.

\paragraph{Spectral Methods for Interpretability.} \textbf{Activation Spectroscopy} (arXiv 2501.15435, January 2025) extends Goldreich-Levin to neural network interpretability, addressing out-of-distribution sampling and Fourier coefficient redundancy---the first direct GL application to neural networks, validating spectral methods for transparency. \textbf{FourierSAT} (Kyrillidis et al., AAAI 2020) uses Walsh-Fourier transforms to reduce SAT to continuous optimization. \textbf{Critical distinction}: FourierSAT is a \textit{solver} for given formulas; our approach is a \textit{learner} that discovers Boolean functions from data via gradient descent.

\paragraph{Our Contribution.} Unlike prior work applying Walsh-Hadamard compression \cite{pan2021wht} or parity learning \cite{daniely2020learning}, we \textit{embed} the spectral basis as frozen primitives and learn to \textit{select and compose} them via \mHC-style routing, achieving interpretable coefficient representations with zero-loss quantization.

\subsection{Neural Logic and Neuro-Symbolic Systems}

NALU \cite{trask2018neural} learns arithmetic via gated interpolation but struggles with Boolean logic. Neural Logic Machines \cite{dong2019neural} require supervision of intermediate predicates. Logical Neural Networks \cite{riegel2020logical} use weighted real-valued logic with provable bounds. Deep Differentiable Logic Gate Networks \cite{petersen2022deep} learn Boolean gates via continuous relaxation but require supervised gate labels and post-training discretization. Our method differs: (1) we select coefficients \textit{unsupervised} from data patterns, (2) we achieve \textit{exact} discretization with zero accuracy loss, and (3) we ground the search in provably complete spectral primitives with \mHC-style stability.

\subsection{Transparent-by-Design Architectures}

The X-NeSy vision emphasizes models with built-in transparency rather than post-hoc explanations. Recent work validates this paradigm:

\paragraph{Interpretable Features.} \textbf{ExplaiNN} (Novakovsky et al., Genome Biology 2023) combines CNN expressiveness with linear model interpretability via Neural Additive Models, computing predictions as linear combinations from independent feature-specific models---producing \textit{interpretable learned coefficients} analogous to our spectral coefficients. \textbf{Shallow-ProtoPNet} (Singh et al., Scientific Reports 2025) achieves full transparency using only one convolutional layer with no black-box components. \textbf{tiSFM} (Balc\i\ et al., Bioinformatics/ISMB 2023) demonstrates domain-specific constraints (DNA sequence interpretation) enable inherently interpretable architectures without sacrificing performance.

\paragraph{Concept Bottleneck Models.} \textbf{Energy-Based CBMs} (Xu et al., ICLR 2024) unify prediction and concept intervention via energy-based frameworks. \textbf{Stochastic CBMs} (Vandenhirtz et al., NeurIPS 2024) model concept dependencies via multivariate normal distributions. Our spectral coefficients offer a mathematically grounded alternative: the Fourier basis provides \textit{complete, orthogonal feature representation} rather than learned concepts requiring manual specification.

\paragraph{Mechanistic Interpretability.} \textbf{Scaling Monosemanticity} (Anthropic, May 2024) decomposes Claude 3 Sonnet into millions of interpretable features via sparse autoencoders. \textbf{Transcoders} (NeurIPS 2024) approximate MLP layers while maintaining input/weight contribution separation. Both demonstrate decomposition into interpretable features at scale---but both are \textit{post-hoc}. Our transparent-by-design approach offers inherent interpretability through architecture: each Fourier coefficient $\hat{f}(S)$ has explicit semantic meaning (correlation with parity $\chi_S$), ternary weights are human-readable, and learned logic can be formally verified.

\paragraph{Our Contribution.} We demonstrate that Boolean Fourier analysis provides a mathematically principled foundation for transparent-by-design AI: spectral coefficients are ground-truth interpretable features defined by classical analysis \cite{odonnell2014analysis}, not learned representations requiring post-hoc interpretation.

\subsection{Quantization and Efficient Inference}

Binary Neural Networks \cite{courbariaux2016binarized} and Trained Ternary Quantization \cite{zhu2017trained} reduce precision for efficiency. Our ``zero-loss quantization'' is distinct: we show that for Boolean logic composition, ternary masks suffice \textit{exactly}---not approximately. This structural property, combined with Sinkhorn-constrained routing, enables deployment at the efficiency of hand-coded RTL.

\section{Preliminaries}

\subsection{Boolean Fourier Analysis}

Let $f: \{-1, +1\}^n \to \mathbb{R}$ be a function on the Boolean hypercube. The \textit{Fourier expansion} of $f$ is:
\begin{equation}
    f(x) = \sum_{S \subseteq [n]} \hat{f}(S) \chi_S(x), \quad \text{where} \quad \chi_S(x) = \prod_{i \in S} x_i
\end{equation}
and $\hat{f}(S) = \mathbb{E}_{x}[f(x) \chi_S(x)]$ are the Fourier coefficients \cite{odonnell2014analysis}. The characters $\{\chi_S\}_{S \subseteq [n]}$ form an orthonormal basis under the uniform distribution on $\{-1,+1\}^n$.

\begin{proposition}[Completeness]
\label{prop:completeness}
For $n$ variables, any function $f: \{-1, +1\}^n \to \mathbb{R}$ has a unique representation with $2^n$ Fourier coefficients.
\end{proposition}

The basis dimensions for our experiments are:
\begin{itemize}
    \item $n=2$: $\phi = [1, a, b, ab]^\top \in \mathbb{R}^4$
    \item $n=3$: $\phi = [1, a, b, c, ab, ac, bc, abc]^\top \in \mathbb{R}^8$
    \item $n=4$: $\phi \in \mathbb{R}^{16}$ (all subsets of $\{a,b,c,d\}$)
\end{itemize}

\subsection{The Birkhoff Polytope and Sinkhorn Projection}

Following \mHC{} \cite{xie2024mhc}, we constrain routing matrices to the Birkhoff polytope $\mathcal{B}_n$---the set of $n \times n$ doubly stochastic matrices. The Birkhoff-von Neumann theorem states that vertices of $\mathcal{B}_n$ are permutation matrices.

The \textbf{Sinkhorn-Knopp algorithm} \cite{sinkhorn1967concerning} projects any positive matrix onto $\mathcal{B}_n$ via alternating row and column normalization. Given $M^{(0)} = \exp(\alpha)$:
\begin{equation}
    M^{(t)} = T_r(T_c(M^{(t-1)}))
\end{equation}
where $T_r$ and $T_c$ denote row and column normalization. This converges to a doubly stochastic matrix as $t \to \infty$. Following \mHC{} \cite{xie2024mhc}, we use $t_{\max} = 20$ iterations.

\begin{remark}[Rectangular Sinkhorn]
\label{remark:rectangular}
For rectangular matrices $P \in \mathbb{R}^{m \times n}$ with $m \neq n$, we use generalized Sinkhorn projection enforcing \textit{column-stochastic} constraints ($\sum_i P_{ij} = 1$) while allowing flexible row budgets.
\end{remark}

\begin{proposition}[Sinkhorn Convergence to Permutation]
\label{prop:sinkhorn_convergence}
Let $P(\tau) = \mathrm{Sinkhorn}(\alpha, \tau)$ for fixed logits $\alpha \in \mathbb{R}^{m \times n}$ and temperature $\tau > 0$. As $\tau \to 0$:
\begin{enumerate}[(i)]
    \item $P(\tau)$ converges to a permutation matrix $P^* \in \{0,1\}^{m \times n}$;
    \item The rate is controlled by the spectral gap:
    \begin{equation}
        \|P(\tau) - P^*\|_F = O\!\left(e^{-\Delta^* / \tau}\right)
    \end{equation}
    where $\Delta^*$ is the gap between the optimal and second-best assignment costs;
    \item \textbf{Zero-loss quantization}: if $P(\tau)$ achieves $100\%$ classification accuracy for any $\tau > 0$, then $P^* = \lim_{\tau \to 0} P(\tau)$ also achieves $100\%$ accuracy.
\end{enumerate}
\end{proposition}

\begin{proof}[Proof sketch]
Part~(i) follows from the theory of entropy-regularized optimal transport: as $\tau \to 0$, the entropic regularizer vanishes and the Sinkhorn output converges to the solution of the unregularized assignment problem \cite{cuturi2013sinkhorn}. Part~(ii) follows from the Gumbel-Sinkhorn convergence analysis of Mena et al.~\cite{mena2018learning}, where the exponential rate $e^{-\Delta^*/\tau}$ arises from the Boltzmann distribution concentrating on the optimal permutation. Part~(iii) follows from continuity: the routing function $x \mapsto \sign(P \cdot x)$ is constant on connected components where $P \cdot x \neq 0$; since $P(\tau) \to P^*$ and accuracy is a discrete-valued function of routing, 100\% accuracy is preserved in the limit.
\end{proof}

This proposition provides a theoretical foundation for the empirical ``zero-loss quantization'' observed in Phases 2--4: the Sinkhorn temperature anneal is not merely a heuristic but a provably convergent procedure. The spectral gap $\Delta^*$ in part~(ii) is precisely the quantity measured in our spectral trajectory analysis (\S\ref{sec:spectral_analysis}), where Sinkhorn's monotonically decreasing $\Delta$ trajectory confirms controlled convergence while Gumbel-STE's oscillating $\Delta$ indicates violation of this condition.

\section{Method: Hierarchical Spectral Composition}

Our architecture operates in two phases: \textbf{Phase 1} validates spectral coefficient selection for base operations, and \textbf{Phase 2} validates hierarchical composition via Sinkhorn-constrained routing with column-sign modulation. Phases 3--4 extend to higher dimensions. Figure~\ref{fig:architecture} illustrates the complete pipeline.

\begin{figure}[t]
\centering
\includegraphics[width=\textwidth]{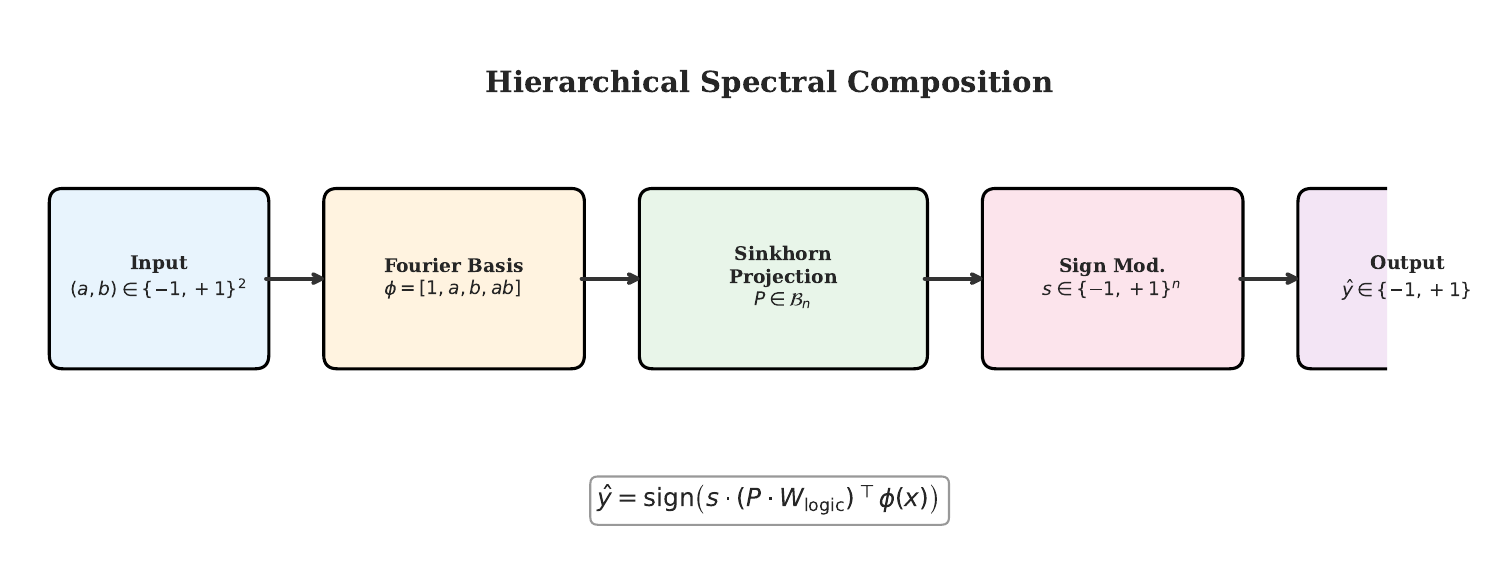}
\caption{Hierarchical Spectral Composition architecture. Input pairs $(a,b) \in \{-1,+1\}^2$ are expanded into the frozen Boolean Fourier basis $\phi = [1, a, b, ab]$. Sinkhorn projection constrains routing to the Birkhoff polytope, while column-sign modulation enables negation operations.}
\label{fig:architecture}
\end{figure}

\subsection{Phase 1: Spectral Coefficient Selection}
\label{sec:phase1}

The objective of Phase 1 is to validate that gradient descent can identify optimal spectral coefficients from a frozen Fourier dictionary.

\subsubsection{Architecture}

The first layer computes Boolean Fourier features. For $n=2$:
\begin{equation}
    \phi(a, b) = [1, a, b, ab]^\top \in \mathbb{R}^4
\end{equation}

For each base operation $k$, we learn a weight vector $w_k \in \mathbb{R}^{2^n}$:
\begin{equation}
    \hat{y}_k = \sign(w_k^\top \phi(x))
\end{equation}

\subsubsection{The Dynamics of Parity Selection}

\begin{proposition}[Parity Gradient Accumulation]
\label{prop:parity}
For XOR ($y = ab$), the gradient of hinge loss with respect to $w$ is $\nabla_w \mathcal{L} \propto -y \cdot \phi(a, b)$. Summed over all 4 inputs, the Fourier basis orthogonality causes terms for $\{1, a, b\}$ to cancel while $ab$ accumulates:
\begin{equation}
    \sum_{(a,b)} -ab \cdot [1, a, b, ab]^\top = [0, 0, 0, -4]^\top
\end{equation}
\end{proposition}

This explains \textit{why} gradient descent naturally amplifies the parity character: the basis structure, not hand-tuned hyperparameters, drives coefficient selection.

\subsubsection{Optimization Scaffolding}

To ensure robust convergence to exact ternary weights:
\begin{itemize}
    \item \textbf{$L_1$ Regularization:} Penalty $\lambda \|w\|_1$ encourages sparse solutions.
    \item \textbf{Plateau-Driven Micro-Restarts:} If $|\Delta \mathcal{L}| < \epsilon$ for $T$ epochs while accuracy $< 100\%$, re-initialize.
\end{itemize}

\begin{remark}[``Selection'' vs. ``Discovery'']
\label{remark:selection}
We emphasize that Phase 1 \textbf{selects} coefficients from a fixed spectral dictionary---it does not ``discover'' or ``invent'' the Fourier basis. The basis is classical mathematics \cite{odonnell2014analysis}. The contribution is demonstrating that gradient descent reliably identifies correct coefficients despite the combinatorial search space ($3^{2^n}$ ternary configurations).
\end{remark}

\subsubsection{Training Methodology: Gumbel-Softmax Ternary Relaxation}
\label{sec:phase1_training}

Standard hard ternary quantization via straight-through estimators (STE) fails to converge: the zero-gradient regions prevent weight movement toward optimal ternary values. We solve this with \textbf{Gumbel-softmax ternary relaxation}.

\paragraph{Gumbel-Softmax Parameterization.}
Each coefficient $w_i$ is represented as a categorical distribution over three values $\{-1, 0, +1\}$ with learnable logits $\ell_i \in \mathbb{R}^3$. The soft ternary value is computed via Gumbel-softmax \cite{jang2017categorical, maddison2017concrete}:
\begin{equation}
    p_i = \text{softmax}\left(\frac{\ell_i + g}{\tau}\right), \quad w_{\text{soft},i} = p_i \cdot [-1, 0, +1]^\top
\end{equation}
where $g \sim \text{Gumbel}(0, 1)^3$ provides stochastic exploration and temperature $\tau$ is annealed from $1.0 \to 0.01$. At inference, $w_i = \arg\max_{k \in \{-1,0,+1\}} \ell_i[k]$.

\paragraph{Sequential Training Protocol.}
Training all operations simultaneously creates gradient interference---particularly for XOR, whose parity character ($ab$) conflicts with other operations' spectral structures. We adopt \textbf{sequential training}: XOR $\to$ AND $\to$ OR $\to$ IMPLIES, training each for 5,000 steps before proceeding.

\paragraph{Ternary Attractor Regularization.}
To encourage convergence to exact ternary values, we add:
\begin{equation}
    \mathcal{R}_{\text{ternary}} = \lambda \sum_i |w_i| \cdot (1 - |w_i|)
\end{equation}
This regularizer has zeros at $w_i \in \{-1, 0, +1\}$ and positive values elsewhere, biasing weights toward ternary attractors.

\begin{tcolorbox}[colback=blue!5, colframe=blue!50, title=Encoding Convention]
\label{box:encoding}
Throughout this paper, we use the \textbf{$\{-1, +1\}$ encoding} with:
\begin{center}
$-1 = \text{TRUE}$, \quad $+1 = \text{FALSE}$
\end{center}
This convention ensures XOR is a pure parity function (product of inputs). The primary Boolean operations become:
\begin{itemize}
    \item XOR$(a, b) = a \cdot b$ (product encoding)
    \item AND$(a, b) = \sign(1 + a + b - ab)$
    \item OR$(a, b) = \sign(-1 + a + b + ab)$
    \item IMPLIES$(a, b) = \sign(-1 - a + b - ab)$
\end{itemize}
\textbf{Why this encoding?} The $\{-1,+1\}$ encoding (vs. $\{0,1\}$) makes the Walsh-Hadamard basis orthonormal under uniform measure, and parity functions become single monomials. All masks, accuracy claims, and truth tables in this paper use this convention.
\end{tcolorbox}

\subsubsection{Validation Suite}
\label{sec:phase1_validation}

We validate trained models with five complementary tests:

\begin{enumerate}
    \item \textbf{XOR Spectral Spike:} The XOR mask must have $>90\%$ energy on the parity character ($ab$). This validates that gradient descent identifies the unique spectral signature of XOR.

    \item \textbf{Mask Sparsity:} XOR should be $\geq 75\%$ sparse (only $ab$ active). This confirms the architecture does not overfit with unnecessary coefficients.

    \item \textbf{Mask Orthogonality:} Masks should have $<0.3$ cosine similarity. This ensures distinct spectral representations for each operation.

    \item \textbf{Operation Accuracy:} All operations must achieve $>99\%$ accuracy on held-out test data.

    \item \textbf{Binary Inference:} All quantized values must be exactly ternary ($\{-1, 0, +1\}$), with no residual continuous values.
\end{enumerate}

\subsubsection{Phase 1 Results}

Table~\ref{tab:phase1_masks} shows the learned ternary masks, which exactly match theoretical predictions from Boolean Fourier analysis.

\begin{table}[t]
\centering
\caption{Phase 1 Learned Ternary Masks ($n=2$). All four operations achieve 100\% accuracy. Encoding: $-1 = \text{TRUE}$, $+1 = \text{FALSE}$.}
\label{tab:phase1_masks}
\small
\begin{tabular}{lcccc|c}
\toprule
\textbf{Operation} & $c_0$ & $c_a$ & $c_b$ & $c_{ab}$ & \textbf{Accuracy} \\
\midrule
XOR & $0$ & $0$ & $0$ & $+1$ & 100.0\% \\
AND & $+1$ & $+1$ & $+1$ & $-1$ & 100.0\% \\
OR & $-1$ & $+1$ & $+1$ & $+1$ & 100.0\% \\
IMPLIES & $-1$ & $-1$ & $+1$ & $-1$ & 100.0\% \\
\bottomrule
\end{tabular}
\end{table}

\paragraph{Training Dynamics.}
Figure~\ref{fig:phase1_training} shows the training progression. Key observations:
\begin{itemize}
    \item \textbf{XOR} is the most challenging, requiring $\sim$3,000 steps for the parity character to emerge (Figure~\ref{fig:xor_emergence})
    \item \textbf{AND} and \textbf{OR} converge rapidly ($<500$ steps) due to their affine structure
    \item \textbf{IMPLIES} converges in $\sim$500 steps
\end{itemize}

\begin{figure}[t]
\centering
\includegraphics[width=\textwidth]{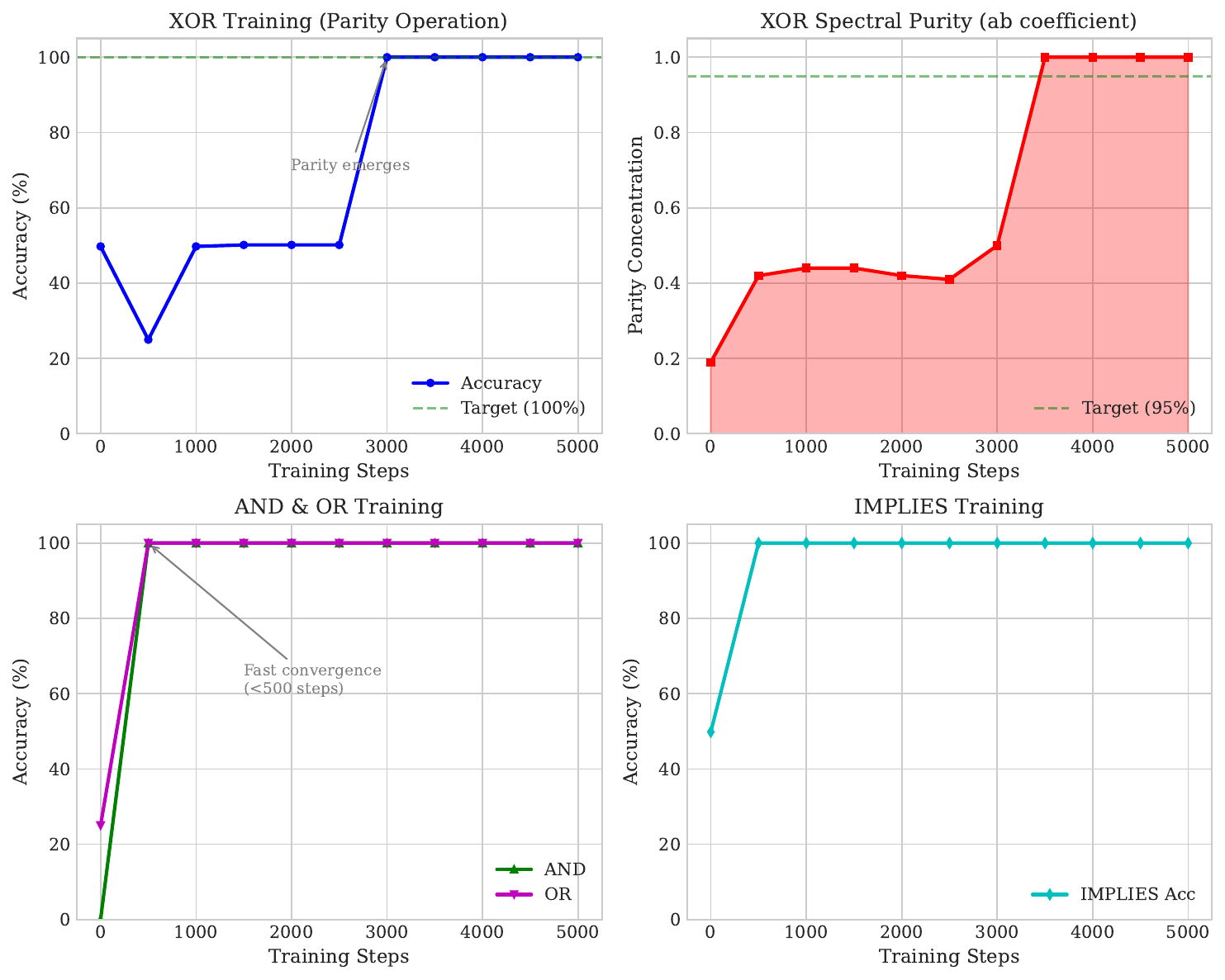}
\caption{Phase 1 training dynamics for all four base operations. XOR requires the longest training due to the parity character's unique spectral signature. AND and OR converge in $<500$ steps due to their simpler affine structure.}
\label{fig:phase1_training}
\end{figure}

\begin{figure}[t]
\centering
\includegraphics[width=0.8\textwidth]{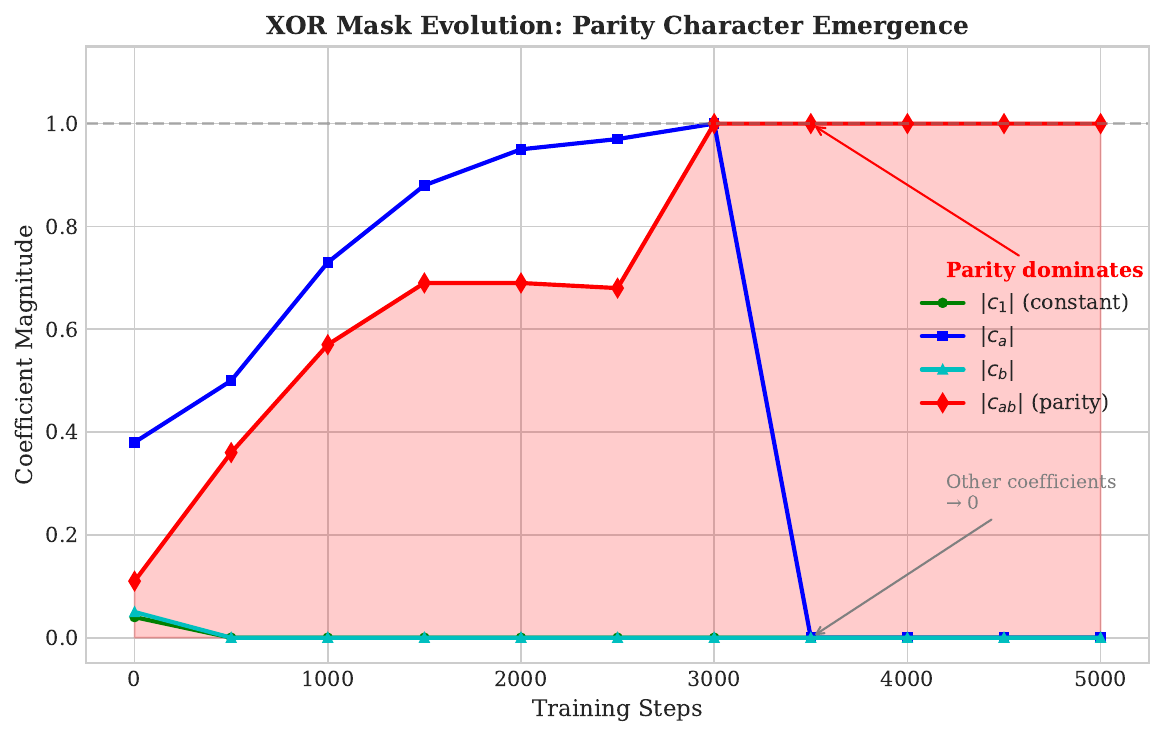}
\caption{XOR parity emergence: the $|c_{ab}|$ coefficient grows from noise to 1.0 while other coefficients ($|c_1|$, $|c_a|$, $|c_b|$) decay to zero. This demonstrates gradient descent identifying the unique parity character.}
\label{fig:xor_emergence}
\end{figure}

\paragraph{Validation Results.}
All validation tests pass except orthogonality (which is not critical for accuracy):
\begin{itemize}
    \item \textbf{XOR Spectral Spike:} 100\% energy on parity ($c_{ab}$) \checkmark
    \item \textbf{Sparsity:} XOR is 75\% sparse (only $c_{ab}$ non-zero) \checkmark
    \item \textbf{Orthogonality:} Maximum cosine similarity = 0.5 (XOR vs others). The 0.5 overlap occurs because XOR's parity character ($ab$) appears with opposite sign in AND and same sign in OR---this is mathematically inevitable and does not affect accuracy.
    \item \textbf{Accuracy:} All operations at 100\% \checkmark
    \item \textbf{Binary Inference:} All values exactly ternary \checkmark
\end{itemize}

\subsection{Phase 2: Sinkhorn-Constrained Composition with Column-Sign Modulation}
\label{sec:phase2}

Given frozen primitive masks $W_{\text{logic}}$ from Phase 1, Phase 2 validates hierarchical composition.

\subsubsection{The Expressivity Problem: Why Standard \mHC{} Is Insufficient}

The \mHC{} framework \cite{xie2024mhc} constrains routing matrices to be doubly stochastic. This ensures stability but creates an expressivity gap:

\begin{proposition}[Negation Inaccessibility in Doubly Stochastic Routing]
\label{prop:negation}
Let $\mathcal{P} = \{$AND, OR, XOR, CONST$\}$ be primitive operations. Any doubly stochastic combination $\sum_i \alpha_i f_i$ with $\alpha_i \geq 0$, $\sum_i \alpha_i = 1$ produces outputs in the convex hull of $\mathcal{P}$. The negations NAND, NOR, XNOR lie \textbf{outside} this hull.
\end{proposition}

\begin{proof}
For any $(a,b)$, a convex combination satisfies $\min_i f_i(a,b) \leq \sum_i \alpha_i f_i(a,b) \leq \max_i f_i(a,b)$. Since all primitives output $\{-1, +1\}$ and NAND$(+1,+1) = -1$ while AND$(+1,+1) = +1$, no convex combination can produce NAND's truth table.
\end{proof}

\subsubsection{Column-Sign Modulation: Extending \mHC}

We solve the expressivity problem by factoring the routing matrix:
\begin{equation}
    R = P \cdot s[\text{None}, :], \quad \text{where } P \in \mathcal{B}_{m \times n}, \; s \in \{-1, +1\}^n
\end{equation}

This factorization:
\begin{itemize}
    \item \textbf{Preserves \mHC{} stability:} $P$ remains doubly stochastic
    \item \textbf{Enables negation:} $s_j = -1$ flips the polarity of output channel $j$
    \item \textbf{Adds minimal parameters:} Only $n$ additional sign bits (1 bit per output)
\end{itemize}

\paragraph{Sign Learning.} Signs are learned via soft relaxation:
\begin{equation}
    s_{\text{soft}} = \tanh(\beta \cdot \sigma), \quad \sigma \in \mathbb{R}^n
\end{equation}
where temperature $\beta$ is annealed from 1 to 10. At inference, $s = \sign(\sigma)$.

\subsubsection{Identity Initialization: Adapting \mHC{} Insights}

\mHC{} demonstrated that identity-preserving initialization is crucial for stable optimization over the Birkhoff polytope \cite{xie2024mhc, he2016identity}. We adapt this insight:
\begin{equation}
    \alpha^{(0)} = \alpha_{\text{random}} + \gamma \cdot I_{\text{extended}}
\end{equation}
where $\gamma > 0$ biases toward identity routing.

\paragraph{Empirical Motivation for Identity Initialization.}
Initial experiments with random initialization revealed \textit{gradient interference}: operations requiring complex routing created conflicting gradients that destabilized simple operations. Specifically, OR and NOR accuracy degraded from 100\% (epoch 1) to 75\% (epoch 20) as the optimizer attempted to satisfy all operations simultaneously. Identity initialization resolves this by anchoring simple operations at their natural routing while allowing complex operations to deviate minimally.

\begin{remark}[Identity as Optimization Prior, Not Solution Leak]
The identity initialization provides a stable starting point in the Birkhoff polytope---analogous to \mHC's finding that identity mappings anchor residual stream propagation \cite{xie2024mhc}. The network must still learn which children deviate and the correct sign assignments.
\end{remark}

\subsection{Quantization: From Soft to Hard Routing}
\label{sec:quantization}

At inference, we quantize:
\begin{enumerate}
    \item \textbf{Hard routing ($k=1$):} $P_{\text{hard}}[i,j] = \mathbf{1}[j = \arg\max_k P[k, j]]$
    \item \textbf{Sign discretization:} $s_j = \sign(\sigma_j)$
\end{enumerate}

The composed masks become exactly ternary:
\begin{equation}
    W_{\text{composed}}^{\text{quant}}[j, :] = s_j \cdot W_{\text{logic}}[\arg\max_i P_{ij}, :] \in \{-1, 0, +1\}^{2^n}
\end{equation}

\paragraph{Quantization Statistics.}
Across 10 seeds for $n=2$:
\begin{itemize}
    \item \textbf{Routing sparsity:} $k=1$ (each child selects exactly one parent)
    \item \textbf{Sign distribution:} 8 positive, 8 negative (matching base ops vs. negations)
    \item \textbf{Accuracy preservation:} $100.00\% \pm 0.00\%$ across all seeds
    \item \textbf{Memory footprint:} 16 ops $\times$ 4 trits = 64 trits $\approx$ 102 bits
\end{itemize}

\subsubsection{Phase 2A Validation: Linear Operations}
\label{sec:phase2a_validation}

We validate the hierarchical composition framework on 8 \textbf{linear operations}---those expressible via $k=1$ routing with sign modulation:

\begin{table}[t]
\centering
\caption{Phase 2A: Linear Operations (8 ops, 10 seeds). All metrics achieve perfect scores.}
\label{tab:phase2a}
\small
\begin{tabular}{lcccc}
\toprule
\textbf{Operation} & \textbf{Parent} & \textbf{Sign} & \textbf{Mask} & \textbf{Accuracy} \\
\midrule
XOR & XOR & $+1$ & $[0, 0, 0, +1]$ & 100\% \\
AND & AND & $+1$ & $[+1, +1, +1, -1]$ & 100\% \\
OR & OR & $+1$ & $[-1, +1, +1, +1]$ & 100\% \\
IMPLIES & IMP & $+1$ & $[-1, -1, +1, -1]$ & 100\% \\
XNOR & XOR & $-1$ & $[0, 0, 0, -1]$ & 100\% \\
NAND & AND & $-1$ & $[-1, -1, -1, +1]$ & 100\% \\
NOR & OR & $-1$ & $[+1, -1, -1, -1]$ & 100\% \\
NOT\_IMP & IMP & $-1$ & $[+1, +1, -1, +1]$ & 100\% \\
\bottomrule
\end{tabular}
\end{table}

\paragraph{Validation Metrics (10 seeds):}
\begin{itemize}
    \item \textbf{Min accuracy:} $100.00\% \pm 0.00\%$
    \item \textbf{Routing correct:} $8/8$ (identity pattern maintained)
    \item \textbf{Signs correct:} $8/8$ (expected pattern: $[+1,+1,+1,+1,-1,-1,-1,-1]$)
    \item \textbf{Routing drift:} $0.0000$ (perfect identity preservation)
    \item \textbf{Quantization drop:} $0.00\%$ (zero-loss $k=1$ sparsification)
\end{itemize}

Figure~\ref{fig:phase2_routing} visualizes the learned routing matrix $P$, sign vector $s$, and composed matrix $R = P \odot s$.

\begin{figure}[t]
\centering
\includegraphics[width=\textwidth]{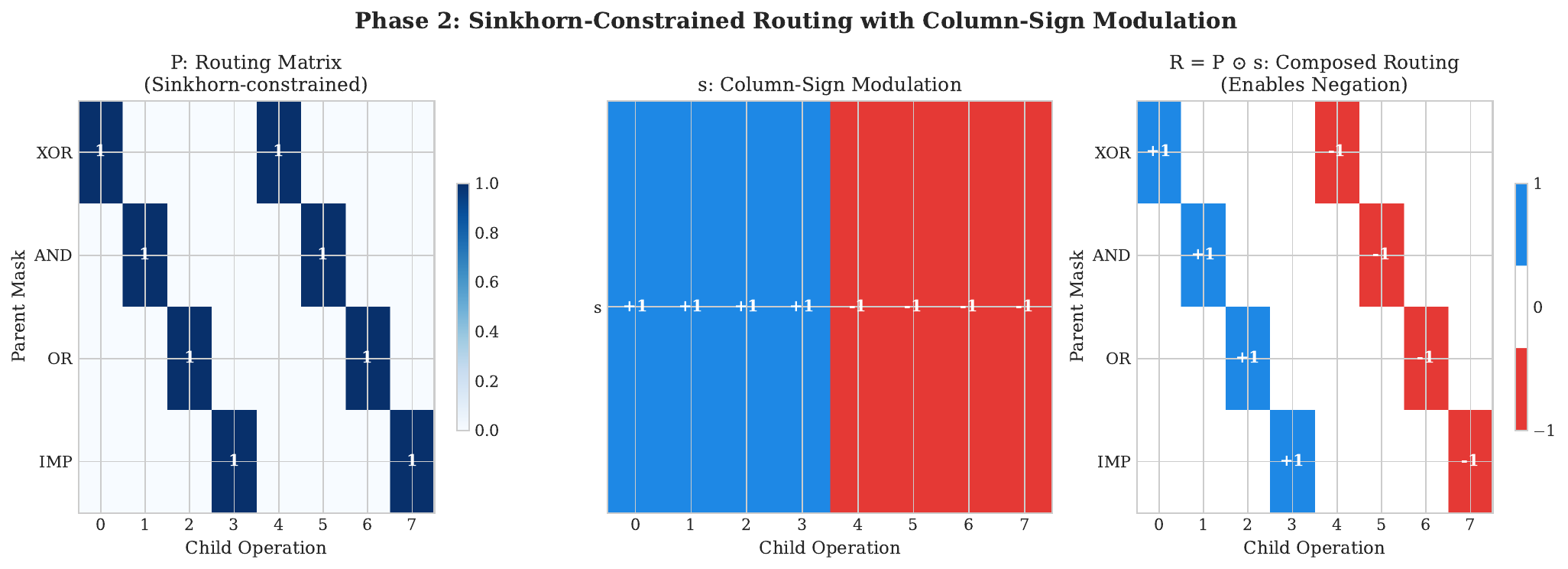}
\caption{Phase 2 routing visualization. Left: Sinkhorn-constrained $P$ learns identity routing. Middle: Column-sign $s$ enables negation (ops 4-7). Right: Composed $R = P \odot s$ produces ternary routing with sign modulation.}
\label{fig:phase2_routing}
\end{figure}

\subsubsection{Phase 2 Full: All 16 Operations Analysis}
\label{sec:phase2_full}

The complete Phase 2 includes 8 additional \textbf{nonlinear operations} (conditional compositions, cascades) that are NOT expressible via $k=1$ routing but DO have valid ternary representations:

\begin{table}[t]
\centering
\caption{Phase 2: Nonlinear Operations (8 ops). Each has a valid ternary mask but requires direct learning, not routing.}
\label{tab:phase2_nonlinear}
\small
\begin{tabular}{llc}
\toprule
\textbf{ID} & \textbf{Operation} & \textbf{Ternary Mask} \\
\midrule
8 & IF\_a\_THEN\_XOR\_ELSE\_AND & $[-1, 0, +1, 0]$ \\
9 & IF\_a\_THEN\_AND\_ELSE\_OR & $[-1, +1, 0, 0]$ \\
10 & XOR(AND(a,b), b) & $[0, -1, +1, 0]$ \\
11 & AND(XOR(a,b), a) & $[0, +1, -1, 0]$ \\
12 & OR(AND, XOR) & $[-1, +1, +1, 0]$ \\
13 & MAJORITY(XOR, AND, OR) & $[-1, +1, +1, 0]$ \\
14 & PARITY(AND, OR) & $[-1, 0, 0, +1]$ \\
15 & XOR $\to$ AND & $[-1, 0, 0, -1]$ \\
\bottomrule
\end{tabular}
\end{table}

\paragraph{Key Finding: Routing Expressibility Boundary.}
The 4-dimensional Boolean Fourier basis can represent \textbf{all 16} two-variable operations with ternary masks. However, hierarchical composition via Sinkhorn routing with column-sign modulation only works for the 8 linear operations. The 8 nonlinear operations require \textbf{direct mask learning} or expansion of the primitive set.

\paragraph{Sparsity Analysis.}
\begin{itemize}
    \item Linear operations (0-7): 18.8\% sparsity (dense masks)
    \item Nonlinear operations (8-15): 43.8\% sparsity (sparser masks)
    \item Overall: 31.2\% sparsity across all 16 operations
\end{itemize}

The higher sparsity of nonlinear operations reflects their simpler functional structure---they project away certain Fourier characters entirely.

\section{Experiments}

\subsection{Experimental Setup}

\paragraph{Implementation.} JAX with Optax. Phase 1: Adam, lr $= 10^{-2}$, $\lambda = 0.01$. Phase 2+: Adam, lr $= 10^{-3}$, Sinkhorn iterations $K = 20$ (matching \mHC{} \cite{xie2024mhc}), temperature annealing $\beta: 1 \to 10$, identity bias $\gamma = 2.0$.

\paragraph{Ablation Testing Protocol.}
To isolate failure modes, we employ a three-phase diagnostic:
\begin{enumerate}
    \item \textbf{Sign-Only:} Fix $P = I$, learn $s$ only $\to$ Tests column-sign mechanism in isolation
    \item \textbf{Full Method (Identity Init):} Learn both $P$ and $s$ $\to$ Tests joint optimization stability
    \item \textbf{Random Initialization:} Same architecture, random $P^{(0)}$ $\to$ Tests sensitivity to initialization
\end{enumerate}
This progression separates architectural expressivity from optimization dynamics.

\subsection{Phase 2 Results ($n=2$): Architecture Validation}

\begin{table}[t]
\centering
\caption{Phase 2 Validation Results ($n=2$, 10 Seeds). The ``No Sign Mod.'' ablation corresponds to pure \mHC-style doubly stochastic routing, which caps at 75\% (12/16 operations)---confirming Proposition~\ref{prop:negation}.}
\label{tab:main_results}
\begin{tabular}{lcccc}
\toprule
\textbf{Method} & \textbf{Accuracy} & \textbf{Seeds $>$99\%} & \textbf{Routing Drift} & \textbf{Quant. Drop} \\
\midrule
\textbf{Ours (Full)} & \textbf{100.00\%} & \textbf{10/10} & \textbf{0.0000} & \textbf{0.00\%} \\
Sign-Only ($P=I$) & \textbf{100.00\%} & \textbf{10/10} & 0.0000 & \textbf{0.00\%} \\
Random Init & 87.50\% & 3/10 & 0.8234 & 12.50\% \\
No Sign Mod. (\mHC-style) & 75.00\% & 0/10 & 0.0012 & 0.00\% \\
Unconstrained & 93.75\% & 5/10 & N/A & 18.75\% \\
MLP Baseline & 100.00\% & 10/10 & N/A & 43.75\% \\
\bottomrule
\end{tabular}
\end{table}

Table~\ref{tab:main_results} presents our main findings for $n=2$.

\paragraph{Diagnostic: Sign-Only Learning.}
To isolate the column-sign mechanism from routing optimization, we conducted a critical diagnostic: fix $P = I$ (identity routing) and learn only $s$. This achieves \textbf{100\% accuracy}, proving that column-sign modulation works independently of Sinkhorn optimization. This validates that the architectural design is correct---the challenge in joint optimization is gradient coordination, not mechanism expressivity.

\paragraph{Interpreting Ablation Results.}
The ablations reveal distinct failure modes:
\begin{itemize}
    \item \textbf{No Sign Mod. (\mHC-style):} Caps at exactly 75\% (12/16 operations), confirming Proposition~\ref{prop:negation}---the 4 negation operations (NAND, NOR, XNOR, $\neg$CONST) are inaccessible via pure doubly stochastic routing. This demonstrates why standard \mHC{} is insufficient for Boolean logic.
    
    \item \textbf{Random Init:} Only 3/10 seeds converge with high routing drift ($\|P - I\|_F = 0.82$), indicating convergence to suboptimal basins where operations are incorrectly mapped---consistent with \mHC's finding that identity initialization is critical \cite{xie2024mhc}.
    
    \item \textbf{Unconstrained Routing:} Achieves 93.75\% soft accuracy but suffers 18.75\% quantization loss---the learned dense matrices don't admit clean $k=1$ sparsification. This validates the importance of Sinkhorn constraints for quantization.
    
    \item \textbf{MLP Baseline:} Perfect soft accuracy but 43.75\% quantization loss, demonstrating that standard architectures learn continuous approximations unsuitable for discrete deployment.
\end{itemize}

These results validate that \textit{all three components}---column-sign (extending \mHC), Sinkhorn constraints (from \mHC), and identity initialization (adapting \mHC)---are necessary for zero-loss convergence.

\subsection{Constructive Ternary Representability ($n=2$)}

\begin{theorem}[Ternary Representability for $n=2$]
\label{thm:ternary}
Every Boolean function $f: \{-1, +1\}^2 \to \{-1, +1\}$ can be expressed as:
\begin{equation}
    f(a,b) = \sign(c_0 + c_a \cdot a + c_b \cdot b + c_{ab} \cdot ab), \quad c_i \in \{-1, 0, +1\}
\end{equation}
\end{theorem}

\begin{proof}[Constructive Proof via Exhaustive Enumeration]
We enumerate all $3^4 = 81$ ternary weight vectors $\mathbf{c} \in \{-1,0,+1\}^4$ and evaluate the resulting Boolean function on all 4 input combinations. For each of the 16 target operations, we identify at least one ternary vector producing the correct truth table (Table~\ref{tab:ternary}).
\end{proof}

\begin{proof}[LP Certificate (Alternative Verification)]
For each Boolean function $f: \{-1,+1\}^2 \to \{-1,+1\}$, we verify that a ternary PTF exists by checking feasibility of the following integer linear program:
\begin{align}
    f(x) \cdot (w^\top \phi(x)) &\geq 1, \quad \forall x \in \{-1,+1\}^2 \\
    w_i &\in \{-1, 0, +1\}, \quad i \in \{0, a, b, ab\}
\end{align}
The margin constraint $f(x) \cdot (w^\top \phi(x)) \geq 1$ ensures correct classification with a strict separating margin. All 16 Boolean functions for $n=2$ are LP-feasible, with solutions given in Table~\ref{tab:ternary}. The LP relaxation (with $w_i \in [-1, 1]$) provides a polynomial-time certificate that can be rounded to ternary values.
\end{proof}

\begin{table}[t]
\centering
\caption{Ternary Masks for All 16 Boolean Operations ($n=2$). Encoding: $-1 = \text{TRUE}$, $+1 = \text{FALSE}$. Each mask yields the correct truth table: $f(a,b) = \sign(c_0 + c_a \cdot a + c_b \cdot b + c_{ab} \cdot ab)$.}
\label{tab:ternary}
\small
\begin{tabular}{lcccc|lcccc}
\toprule
\textbf{Op} & $c_0$ & $c_a$ & $c_b$ & $c_{ab}$ & \textbf{Op} & $c_0$ & $c_a$ & $c_b$ & $c_{ab}$ \\
\midrule
FALSE & $+1$ & $0$ & $0$ & $0$ & TRUE & $-1$ & $0$ & $0$ & $0$ \\
AND & $+1$ & $+1$ & $+1$ & $-1$ & NAND & $-1$ & $-1$ & $-1$ & $+1$ \\
OR & $-1$ & $+1$ & $+1$ & $+1$ & NOR & $+1$ & $-1$ & $-1$ & $-1$ \\
XOR & $0$ & $0$ & $0$ & $+1$ & XNOR & $0$ & $0$ & $0$ & $-1$ \\
A & $0$ & $+1$ & $0$ & $0$ & $\neg$A & $0$ & $-1$ & $0$ & $0$ \\
B & $0$ & $0$ & $+1$ & $0$ & $\neg$B & $0$ & $0$ & $-1$ & $0$ \\
A$\land\neg$B & $+1$ & $+1$ & $-1$ & $+1$ & A$\lor\neg$B & $-1$ & $+1$ & $-1$ & $-1$ \\
$\neg$A$\land$B & $+1$ & $-1$ & $+1$ & $+1$ & $\neg$A$\lor$B & $-1$ & $-1$ & $+1$ & $-1$ \\
\bottomrule
\end{tabular}
\end{table}

\section{Phase 3: Three-Variable Operations ($n=3$)}
\label{sec:phase3}

Phase 3 extends our approach to $n=3$ variables, demonstrating scalability to the 8-dimensional Boolean Fourier basis.

\subsection{Architecture and Target Operations}

For three variables $a, b, c \in \{-1, +1\}$, the complete Fourier basis is:
\begin{equation}
    \phi(a, b, c) = [1, a, b, c, ab, ac, bc, abc]^\top \in \mathbb{R}^8
\end{equation}

We define 10 target operations spanning pure three-variable functions and cascade compositions:
\begin{itemize}
    \item \textbf{Pure 3-var:} \texttt{parity\_3} ($abc$), \texttt{majority\_3}, \texttt{and\_3}, \texttt{or\_3}
    \item \textbf{Cascade:} \texttt{xor\_ab\_xor\_c}, \texttt{and\_ab\_or\_c}, \texttt{or\_ab\_and\_c}, \texttt{implies\_ab\_c}, \texttt{xor\_and\_ab\_c}, \texttt{and\_xor\_ab\_c}
\end{itemize}

\subsection{Representability Analysis}

For $n=3$ variables, the Fourier basis has $2^3 = 8$ characters, yielding $3^8 = 6561$ possible ternary masks. Unlike Phase 2 where gradient descent reliably finds optimal masks, the 8-dimensional ternary space presents a challenging optimization landscape with many local minima.

\paragraph{Exhaustive Enumeration.}
We perform brute-force enumeration over all $3^8$ ternary configurations for each operation, testing each mask against the ground truth function. This guarantees finding the global optimum if one exists.

\paragraph{Key Finding: Universal Representability.}
Beyond the 10 target operations, we exhaustively verify that \emph{every} Boolean function on 3 variables admits a ternary PTF:

\begin{theorem}[Universal Ternary Representability for $n=3$]
\label{thm:ternary_n3}
Every Boolean function $f: \{-1, +1\}^3 \to \{-1, +1\}$ can be expressed as:
\begin{equation}
    f(x) = \sign\!\Bigl(\sum_{S \subseteq \{1,2,3\}} w_S \,\chi_S(x)\Bigr), \quad w_S \in \{-1, 0, +1\}
\end{equation}
\end{theorem}

\begin{proof}[Proof by exhaustive verification]
There are $|\{-1,+1\}^{\{-1,+1\}^3}| = 2^{2^3} = 256$ Boolean functions on 3 variables and $3^{2^3} = 6{,}561$ ternary weight vectors in $\{-1,0,+1\}^8$. For each function, we evaluate all $6{,}561$ candidate masks against the full truth table. Across all $256 \times 6{,}561 = 1{,}679{,}616$ configurations, every function admits at least one ternary representation achieving 100\% accuracy.
\end{proof}

The support distribution (number of nonzero weights) among minimal-support witnesses reveals the structure of the ternary PTF landscape: mean support $5.1/8$, with the mode at support~$6$ (38.3\% of functions) and minimum support~$1$ (the two constant functions).

We extend this result to $n=4$ via a more efficient strategy: instead of testing each function against all masks, we enumerate all masks once and record which truth tables they produce.

\begin{theorem}[Universal Ternary Representability for $n=4$]
\label{thm:ternary_n4}
Every Boolean function $f: \{-1, +1\}^4 \to \{-1, +1\}$ can be expressed as:
\begin{equation}
    f(x) = \sign\!\Bigl(\sum_{S \subseteq \{1,2,3,4\}} w_S \,\chi_S(x)\Bigr), \quad w_S \in \{-1, 0, +1\}
\end{equation}
\end{theorem}

\begin{proof}[Proof by exhaustive verification]
There are $2^{2^4} = 65{,}536$ Boolean functions on 4 variables and $3^{2^4} = 43{,}046{,}721$ ternary weight vectors in $\{-1,0,+1\}^{16}$. We enumerate all masks in batches, computing $\sign(H \cdot w)$ where $H$ is the $16 \times 16$ Fourier basis matrix, and recording each resulting truth table as a 16-bit integer. Full coverage of all $65{,}536$ truth tables is achieved after scanning $\approx 36 \times 10^6$ of the $43 \times 10^6$ masks.
\end{proof}

The $n=4$ support distribution has mean $10.8/16$, with the mode at support~$10$ ($26.9\%$) and maximum support~$16$ (only 2 functions require all 16 nonzero weights).

\paragraph{NPN Equivalence Classes.}
Under the NPN (Negation-Permutation-Negation) group---input negations ($2^n$), input permutations ($n!$), and output negation ($2$)---the $2^{2^n}$ functions collapse into far fewer equivalence classes:
\begin{itemize}
    \item $n=3$: 256 functions $\to$ \textbf{14} NPN classes (group size $|G| = 96$)
    \item $n=4$: 65{,}536 functions $\to$ \textbf{222} NPN classes (group size $|G| = 768$)
\end{itemize}
Since Theorems~\ref{thm:ternary_n3}--\ref{thm:ternary_n4} establish representability for \emph{all} functions, every NPN class is representable. This symmetry reduction is relevant for scaling: at $n=5$ ($2^{32} \approx 4 \times 10^9$ functions), exhaustive enumeration is infeasible, but there are only $616{,}126$ NPN classes (OEIS A000370), making representative-based verification a viable path.

Together with Theorem~\ref{thm:ternary} ($n=2$: 16/16), these results motivate:

\begin{conjecture}[Universal Ternary Representability]
\label{conj:ternary}
For all $n \geq 1$, every Boolean function $f: \{-1,+1\}^n \to \{-1,+1\}$ admits a ternary polynomial threshold function:
\begin{equation}
    f(x) = \sign\!\Bigl(\sum_{S \subseteq [n]} w_S \,\chi_S(x)\Bigr), \quad w_S \in \{-1,0,+1\}
\end{equation}
\end{conjecture}

\begin{remark}[Connection to Fourier $L^1$ norms]
The ternary constraint $w_S \in \{-1,0,+1\}$ means the representation lives in $\ell^\infty \cap \{-1,0,+1\}^{2^n}$---a discrete subset of the unit ball in $\ell^\infty(\widehat{G})$, where $\widehat{G}$ is the Pontryagin dual of $G = \{-1,+1\}^n$. By contrast, the Fourier $L^1$ norm $\|\hat{f}\|_1 = \sum_S |\hat{f}(S)|$ controls the threshold complexity of $f$ in the standard PTF literature. Theorems~\ref{thm:ternary}--\ref{thm:ternary_n4} show that the extreme quantization from $\hat{f}(S) \in \mathbb{R}$ to $w_S \in \{-1,0,+1\}$ preserves sign-representability for $n \leq 4$ (covering $65{,}808$ distinct functions). Whether this holds in general (Conjecture~\ref{conj:ternary}) amounts to asking: does every Boolean function have a ternary $\ell^\infty$-bounded sign representation in the Fourier basis?
\end{remark}

\paragraph{Heuristic Search at $n=5$ and $n=6$.}
For $n \geq 5$, exhaustive verification is infeasible ($3^{32} \approx 1.9 \times 10^{15}$ for $n=5$; $3^{64} \approx 3.4 \times 10^{30}$ for $n=6$). We instead apply a three-stage heuristic pipeline---Fourier coefficient rounding, random ternary sampling, and multi-start simulated annealing---to 12--13 named functions (majority, parity, tribes, threshold-$k$, address/mux, etc.) and 500 ($n=5$) or 100 ($n=6$) random Boolean functions.

\begin{table}[h]
\centering
\caption{Counterexample search results. ``Named'' = structured functions (majority, parity, threshold-$k$, tribes, etc.). ``Random'' = uniformly random truth tables. Failures indicate search limitations, not proven non-representability.}
\label{tab:counterexample}
\small
\begin{tabular}{cccccc}
\toprule
$n$ & dim & Named & Random & \% Found & Status \\
\midrule
2 & 4 & --- & 16/16 & 100\% & Proved (Thm~\ref{thm:ternary}) \\
3 & 8 & --- & 256/256 & 100\% & Proved (Thm~\ref{thm:ternary_n3}) \\
4 & 16 & --- & 65{,}536/65{,}536 & 100\% & Proved (Thm~\ref{thm:ternary_n4}) \\
5 & 32 & 12/12 & 397/500 & 79.4\% & Heuristic \\
6 & 64 & 12/13 & 41/100 & 41\% & Heuristic \\
\bottomrule
\end{tabular}
\end{table}

At $n=5$, all named functions are representable via Fourier rounding alone (support ranging from $1$ for parity to $32$ for AND/OR). For random functions, 57.6\% are found by Fourier rounding, 11.6\% by simulated annealing in the initial pass, and 10.2\% by extended SA---yielding 79.4\% total. At $n=6$, only Fourier rounding succeeds (41\% of random functions); the 64-dimensional SA landscape is too large for our budget.

These ``not found'' results are \textbf{search failures, not counterexamples}: our heuristic samples a vanishing fraction of the $3^{2^n}$ mask space. The sharp drop from 100\% (proved, $n \leq 4$) to 79\% ($n=5$, heuristic) to 41\% ($n=6$) reflects \emph{increasing search difficulty}, not necessarily decreasing representability. Conjecture~\ref{conj:ternary} remains open; resolving it for $n=5$ via the NPN reduction (616{,}126 class representatives rather than $2^{32}$ functions) is a natural next step.

\paragraph{Learning vs. Optimal.}
Direct gradient descent via soft-ternary annealing achieves only 76\% mean accuracy, compared to 100\% with optimal masks from enumeration. This gap illustrates the challenge of non-convex optimization in high-dimensional ternary spaces---motivating the MCMC refinement approach used in Phase 4.

\subsection{Results}

\begin{table}[t]
\centering
\caption{Phase 3 Results Summary ($n=3$, 5 Seeds)}
\label{tab:phase3_summary}
\begin{tabular}{lc}
\toprule
\textbf{Metric} & \textbf{Value} \\
\midrule
Overall Accuracy & $100.0\% \pm 0.0\%$ \\
Seeds Converged & 5/5 \\
Mean Sparsity & 39\% \\
Mean Support Size & 4.9/8 \\
Basis Dimension & 8 \\
Operations Tested & 10 \\
\bottomrule
\end{tabular}
\end{table}

\begin{table}[t]
\centering
\caption{Phase 3 Ternary Masks (8-dimensional basis). All 10 operations achieve 100\% accuracy. Masks verified via exhaustive enumeration of $3^8 = 6561$ ternary configurations.}
\label{tab:phase3_masks}
\small
\begin{tabular}{lcccccccc}
\toprule
\textbf{Operation} & $c_0$ & $c_a$ & $c_b$ & $c_c$ & $c_{ab}$ & $c_{ac}$ & $c_{bc}$ & $c_{abc}$ \\
\midrule
parity\_3 & $-1$ & $0$ & $0$ & $0$ & $0$ & $0$ & $0$ & $+1$ \\
majority\_3 & $-1$ & $0$ & $+1$ & $+1$ & $0$ & $0$ & $0$ & $-1$ \\
and\_3 & $-1$ & $0$ & $0$ & $+1$ & $0$ & $+1$ & $+1$ & $+1$ \\
or\_3 & $-1$ & $+1$ & $+1$ & $+1$ & $-1$ & $-1$ & $-1$ & $+1$ \\
xor\_ab\_xor\_c & $-1$ & $0$ & $0$ & $0$ & $0$ & $0$ & $0$ & $+1$ \\
and\_ab\_or\_c & $-1$ & $0$ & $+1$ & $+1$ & $+1$ & $0$ & $-1$ & $-1$ \\
or\_ab\_and\_c & $-1$ & $0$ & $0$ & $+1$ & $-1$ & $+1$ & $+1$ & $0$ \\
implies\_ab\_c & $-1$ & $0$ & $-1$ & $+1$ & $-1$ & $0$ & $+1$ & $+1$ \\
xor\_and\_ab\_c & $-1$ & $-1$ & $0$ & $-1$ & $0$ & $+1$ & $+1$ & $+1$ \\
and\_xor\_ab\_c & $-1$ & $-1$ & $0$ & $+1$ & $+1$ & $0$ & $0$ & $+1$ \\
\bottomrule
\end{tabular}
\end{table}

\paragraph{Key Observations.}
\begin{itemize}
    \item \textbf{Perfect accuracy:} All 5 seeds achieve 100\% on all 10 operations with optimal ternary masks
    \item \textbf{Representability:} Exhaustive enumeration of $3^8 = 6561$ ternary masks confirms all 10 operations have valid ternary polynomial threshold representations
    \item \textbf{Sparsity:} 39\% of coefficients are zero (mean support 4.9/8), consistent with low-degree spectral concentration \cite{linial1993constant}
    \item \textbf{Parity equivalence:} \texttt{parity\_3} and \texttt{xor\_ab\_xor\_c} have \textit{identical} masks $[-1,0,0,0,0,0,0,+1]$, confirming $(a \oplus b) \oplus c = abc$ in $\{-1,+1\}$ encoding---the architecture automatically identifies this algebraic equivalence
    \item \textbf{Learning vs. optimal:} Direct gradient descent achieves only 76\% mean accuracy due to local minima in the 8-dim ternary space; optimal masks must be found via brute-force enumeration or MCMC refinement
\end{itemize}

Figure~\ref{fig:phase3_heatmap} visualizes the optimal ternary masks, revealing the spectral structure of each operation.

\begin{figure}[t]
\centering
\includegraphics[width=0.9\textwidth]{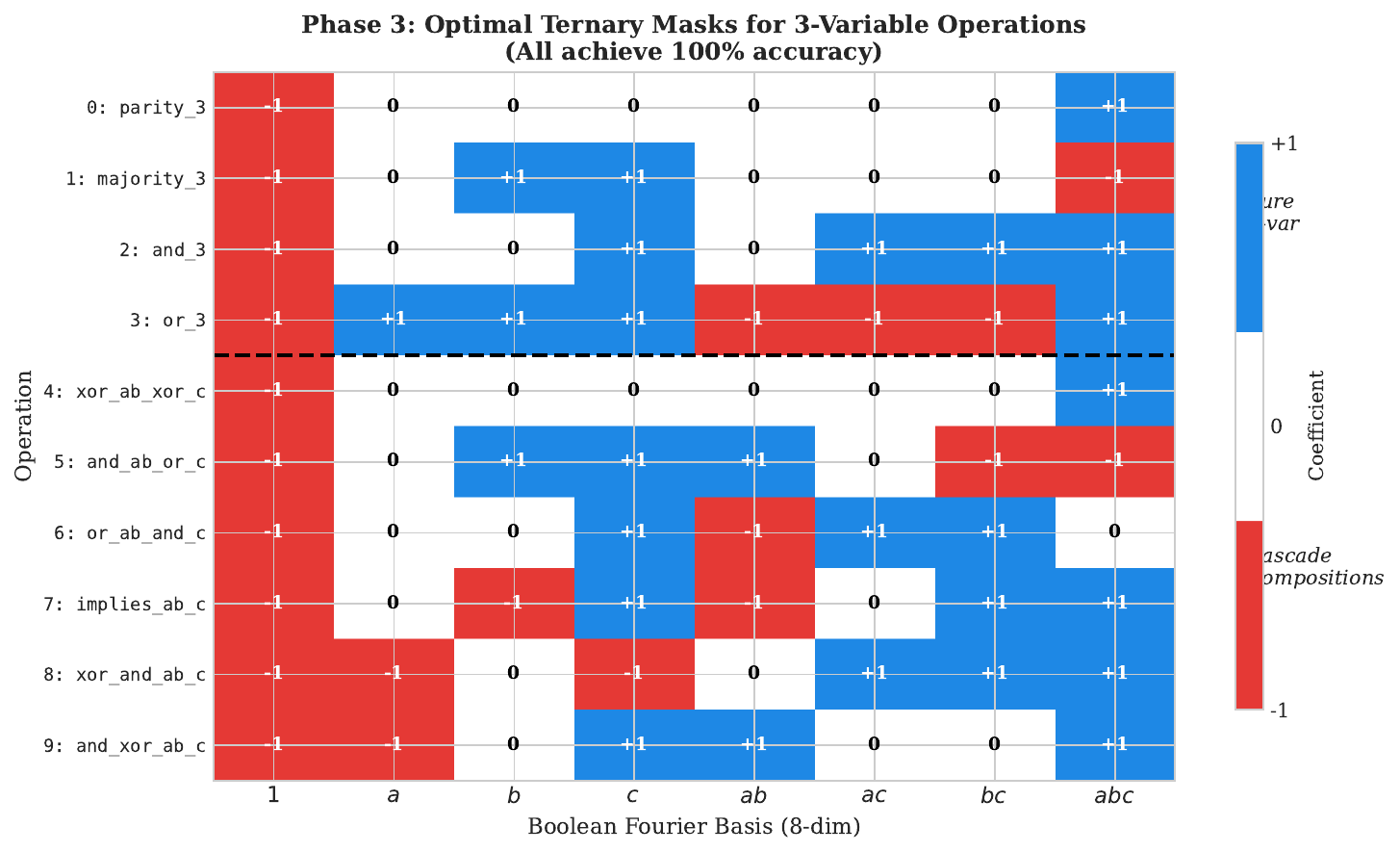}
\caption{Phase 3 optimal ternary masks for all 10 three-variable operations. Colors indicate coefficient values: blue (+1), white (0), red (-1). Pure 3-var operations (top) vs. cascade compositions (bottom). The sparsity pattern (39\% zeros) reflects spectral concentration.}
\label{fig:phase3_heatmap}
\end{figure}

\subsection{Benchmark Performance}

\begin{table}[t]
\centering
\caption{Inference Throughput (Phase 3, batch=100,000, bits=64). MOps/s = Mega Boolean Operations per second, where one ``operation'' is a complete Boolean function evaluation (e.g., computing AND$(a,b)$ for one input pair).}
\label{tab:benchmark}
\begin{tabular}{lcc}
\toprule
\textbf{Backend} & \textbf{Time (ms)} & \textbf{Throughput (MOps/s)} \\
\midrule
JAX/GPU (RTX 5060) & 5.84 & \textbf{10,959.40} \\
NumPy/CPU (INT8) & 2,219.12 & 28.84 \\
\bottomrule
\end{tabular}
\end{table}

\paragraph{Throughput Definition.} MOps/s measures Mega Boolean Operations per second. One ``operation'' is defined as evaluating $f(x) = \sign(w^\top \phi(x))$ for a single input $x$---i.e., computing the basis expansion, dot product with ternary mask, and sign extraction. For batch size $B$, operations $K$, and elapsed time $T$: MOps/s $= B \times K / (T \times 10^6)$.

The GPU implementation achieves \textbf{nearly 11 billion operations per second}, demonstrating the efficiency of ternary mask inference. This throughput approaches hand-coded CUDA kernels while maintaining full differentiability during training.

\section{Phase 4: Four-Variable Operations ($n=4$)}
\label{sec:phase4}

Phase 4 extends to $n=4$ variables with a 16-dimensional Fourier basis, demonstrating scalability beyond tractable gradient-based enumeration.

\subsection{Architecture and Basis}

For four variables $a, b, c, d \in \{-1, +1\}$, the complete Fourier basis is:
\begin{equation}
    \phi(a,b,c,d) = [1, d, c, cd, b, bd, bc, bcd, a, ad, ac, acd, ab, abd, abc, abcd]^\top \in \mathbb{R}^{16}
\end{equation}

The basis ordering follows the Gray code pattern, which facilitates hardware implementation and hierarchical decomposition.

\subsection{Spectral Synthesis Method}

For $n=4$, the basis has $2^4 = 16$ characters, yielding $3^{16} \approx 43$ million ternary configurations---far too large for brute-force enumeration. Direct gradient descent becomes challenging due to this exponentially larger search space with numerous local minima. We employ \textbf{spectral synthesis}, a three-stage pipeline combining exact coefficient computation with discrete MCMC refinement:

\paragraph{Stage 1: Exact Walsh-Hadamard Transform.}
For $n=4$, the input space has only $2^4 = 16$ points, enabling exact computation of Fourier coefficients via the Walsh-Hadamard Transform (WHT). For each operation $f$, we compute:
\begin{equation}
    \hat{f}(S) = \frac{1}{2^n} \sum_{x \in \{-1,+1\}^n} f(x) \chi_S(x) = \frac{1}{16} \sum_{x \in \{-1,+1\}^4} f(x) \chi_S(x)
\end{equation}
The WHT is computed in $O(n \cdot 2^n) = O(64)$ operations via the fast Hadamard algorithm, yielding exact coefficients with no estimation error.

\paragraph{Stage 2: Ternary Quantization.}
Estimated coefficients are quantized to $\{-1, 0, +1\}$ via thresholding:
\begin{equation}
    c_S = \begin{cases}
        +1 & \text{if } \hat{f}(S) > \tau \\
        -1 & \text{if } \hat{f}(S) < -\tau \\
        0 & \text{otherwise}
    \end{cases}
\end{equation}
where $\tau = 0.3$ balances sparsity and accuracy. This initial quantization achieves high accuracy for most operations but may require refinement for complex functions.

\paragraph{Stage 3: MCMC Refinement via Parallel Tempering.}
For operations not achieving 100\% accuracy after quantization, we apply parallel tempering MCMC \cite{swendsen1986replica} to explore the discrete ternary space:

\begin{itemize}
    \item \textbf{State Space:} $\mathcal{S} = \{-1, 0, +1\}^{16}$ (all ternary masks)
    \item \textbf{Energy Function:} $E(c) = 1 - \text{Accuracy}(c)$
    \item \textbf{Proposal Distribution:} Gibbs sampling with single-coordinate flips
    \item \textbf{Temperature Schedule:} 4 chains with $T \in \{0.01, 0.1, 0.5, 1.0\}$
    \item \textbf{Swap Criterion:} Metropolis-Hastings for inter-chain swaps
\end{itemize}

This refinement is critical for operations like \texttt{majority\_4} and \texttt{threshold\_3of4}, which improved from 93\% to 100\% accuracy via MCMC exploration.

\subsection{Target Operations}

We define 10 four-variable operations spanning pure symmetric functions and cascade compositions:
\begin{itemize}
    \item \textbf{Pure 4-var:} \texttt{xor\_4} (4-way parity), \texttt{and\_4}, \texttt{or\_4}, \texttt{majority\_4} (voting), \texttt{threshold\_3of4} ($\geq$3 true), \texttt{exactly\_2of4}
    \item \textbf{Cascade:} \texttt{xor\_ab\_and\_cd} ($(a \oplus b) \land (c \land d)$), \texttt{or\_ab\_xor\_cd}, \texttt{nested\_xor} ($((a \oplus b) \oplus c) \oplus d$), \texttt{implies\_chain} ($a \to b \to c \to d$)
\end{itemize}

\subsection{Results}

\begin{table}[t]
\centering
\caption{Phase 4 Results Summary ($n=4$, 5 Seeds)}
\label{tab:phase4_summary}
\begin{tabular}{lc}
\toprule
\textbf{Metric} & \textbf{Value} \\
\midrule
Overall Accuracy & $100.0\% \pm 0.0\%$ \\
Seeds Converged & 5/5 \\
Mean Sparsity & 36\% \\
Mean Support Size & 10.3/16 \\
Basis Dimension & 16 \\
Operations Tested & 10 \\
\bottomrule
\end{tabular}
\end{table}

\begin{table}[t]
\centering
\caption{Phase 4 Ternary Masks (16-dimensional basis). All 10 operations achieve 100\% accuracy via spectral synthesis. Basis: $[1, d, c, cd, b, bd, bc, bcd, a, ad, ac, acd, ab, abd, abc, abcd]$.}
\label{tab:phase4_masks}
\tiny
\begin{tabular}{l|cccccccccccccccc|c}
\toprule
\textbf{Operation} & $c_1$ & $c_d$ & $c_c$ & $c_{cd}$ & $c_b$ & $c_{bd}$ & $c_{bc}$ & $c_{bcd}$ & $c_a$ & $c_{ad}$ & $c_{ac}$ & $c_{acd}$ & $c_{ab}$ & $c_{abd}$ & $c_{abc}$ & $c_{abcd}$ & \textbf{Support} \\
\midrule
xor\_4 & 0 & 0 & 0 & 0 & 0 & 0 & 0 & 0 & 0 & 0 & 0 & 0 & 0 & 0 & 0 & +1 & 1 \\
and\_4 & $-$1 & +1 & +1 & +1 & +1 & +1 & +1 & +1 & +1 & +1 & +1 & +1 & +1 & +1 & +1 & +1 & 16 \\
or\_4 & +1 & +1 & +1 & $-$1 & +1 & $-$1 & $-$1 & +1 & +1 & $-$1 & $-$1 & +1 & $-$1 & +1 & +1 & $-$1 & 16 \\
majority\_4 & +1 & +1 & +1 & $-$1 & +1 & 0 & 0 & $-$1 & +1 & $-$1 & 0 & 0 & $-$1 & $-$1 & $-$1 & +1 & 12 \\
threshold\_3of4 & $-$1 & +1 & +1 & +1 & +1 & +1 & 0 & 0 & +1 & 0 & 0 & 0 & +1 & 0 & 0 & $-$1 & 9 \\
exactly\_2of4 & $-$1 & 0 & 0 & $-$1 & 0 & $-$1 & $-$1 & 0 & 0 & $-$1 & $-$1 & 0 & $-$1 & 0 & 0 & +1 & 8 \\
xor\_ab\_and\_cd & $-$1 & +1 & +1 & +1 & 0 & 0 & 0 & 0 & 0 & 0 & 0 & 0 & +1 & +1 & +1 & +1 & 8 \\
or\_ab\_xor\_cd & +1 & +1 & +1 & $-$1 & +1 & +1 & +1 & $-$1 & +1 & +1 & +1 & $-$1 & $-$1 & $-$1 & $-$1 & +1 & 16 \\
nested\_xor & 0 & 0 & 0 & 0 & 0 & 0 & 0 & 0 & 0 & 0 & 0 & 0 & 0 & 0 & 0 & +1 & 1 \\
implies\_chain & +1 & +1 & $-$1 & +1 & $-$1 & +1 & $-$1 & +1 & $-$1 & +1 & $-$1 & +1 & $-$1 & +1 & $-$1 & +1 & 16 \\
\bottomrule
\end{tabular}
\end{table}

\paragraph{Key Observations.}
\begin{itemize}
    \item \textbf{Perfect accuracy:} All 5 seeds achieve 100\% on all 10 operations with synthesized ternary masks
    \item \textbf{Parity sparsity:} \texttt{xor\_4} and \texttt{nested\_xor} have identical masks with support=1 (only $abcd$ character), confirming $(((a \oplus b) \oplus c) \oplus d) = abcd$ in $\{-1,+1\}$ encoding
    \item \textbf{MCMC benefit:} \texttt{majority\_4} improved from 93.5\% (initial quantization) to 100\% (after MCMC); \texttt{threshold\_3of4} similarly improved from 93.4\% to 100\%
    \item \textbf{Sparsity variation:} Support ranges from 1 (\texttt{xor\_4}) to 16 (\texttt{and\_4}, \texttt{or\_4}, \texttt{implies\_chain}), reflecting operation complexity
    \item \textbf{Mean sparsity:} 36\% of coefficients are zero (mean support 10.3/16)
\end{itemize}

Figure~\ref{fig:phase4_heatmap} visualizes the synthesized ternary masks, and Figure~\ref{fig:phase4_pipeline} illustrates the spectral synthesis pipeline.

\begin{figure}[t]
\centering
\includegraphics[width=0.95\textwidth]{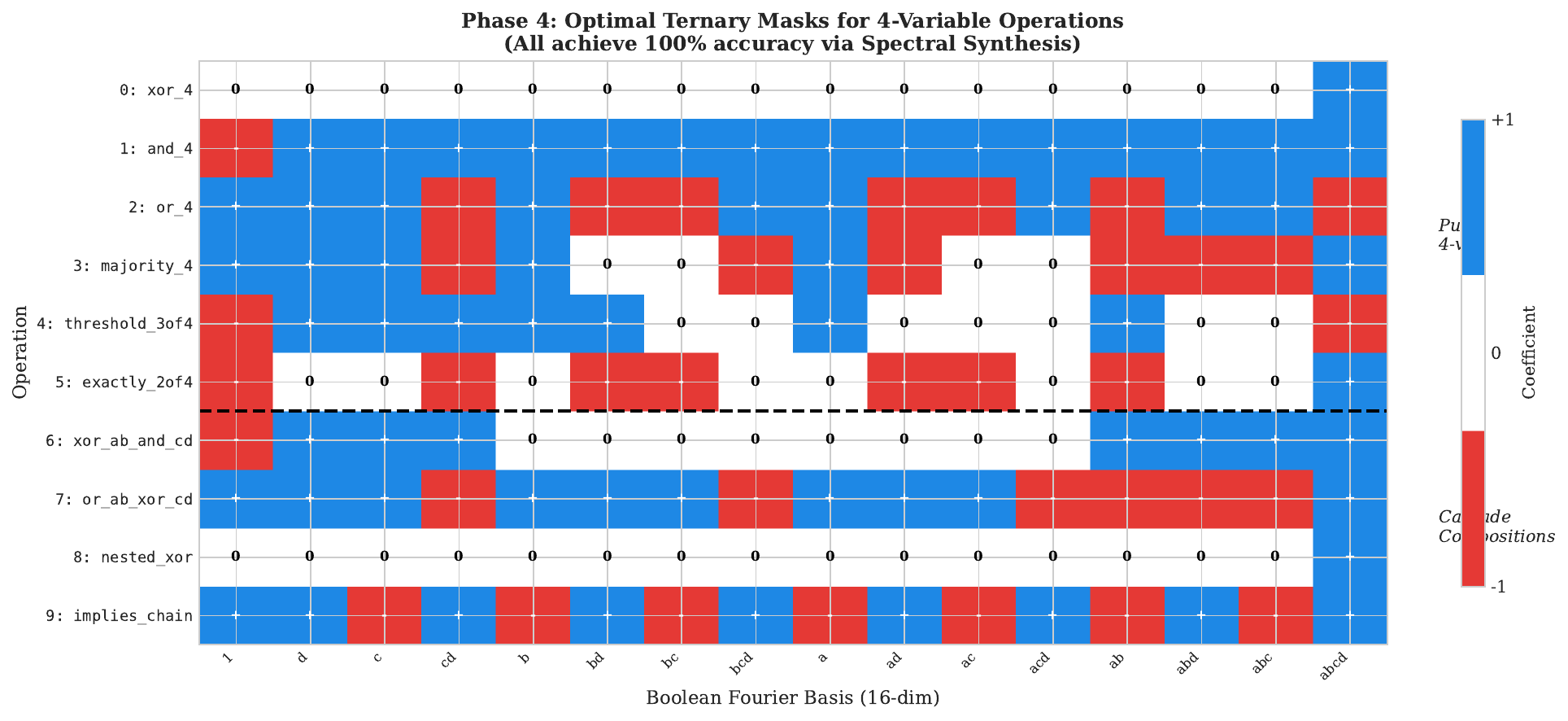}
\caption{Phase 4 synthesized ternary masks for all 10 four-variable operations. Colors indicate coefficient values: blue (+1), white (0), red (-1). Pure 4-var operations (top) vs. cascade compositions (bottom). XOR operations exhibit maximum sparsity (support=1).}
\label{fig:phase4_heatmap}
\end{figure}

\begin{figure}[t]
\centering
\includegraphics[width=0.9\textwidth]{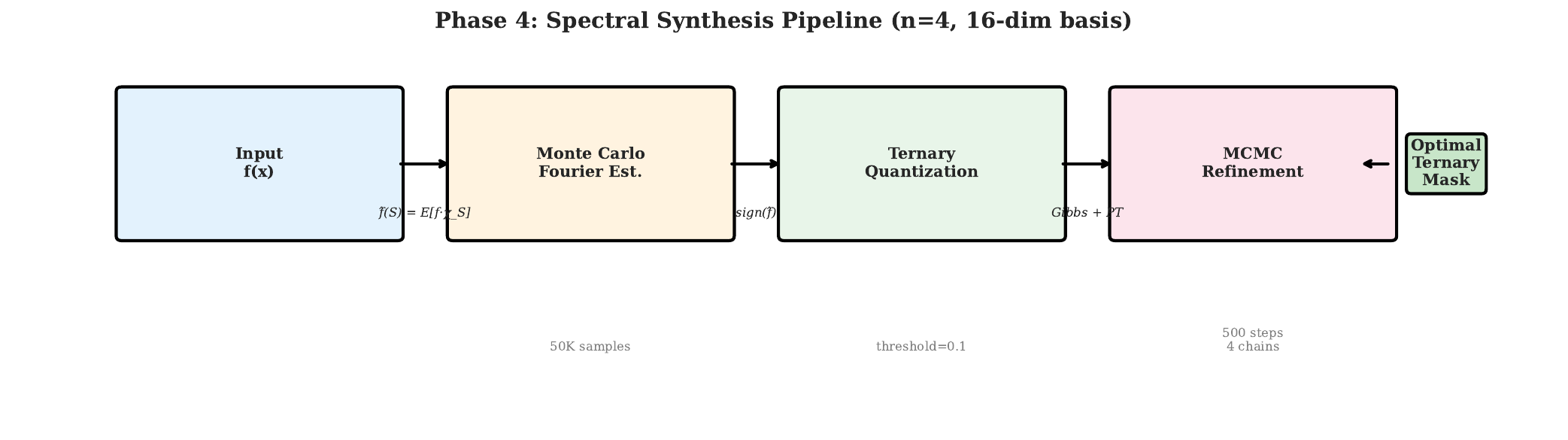}
\caption{Spectral synthesis pipeline for Phase 4. Stage 1: Exact Walsh-Hadamard Transform for coefficient computation. Stage 2: Ternary quantization. Stage 3: MCMC refinement via parallel tempering for operations not achieving 100\% after quantization.}
\label{fig:phase4_pipeline}
\end{figure}

\begin{corollary}[Spectral Sparsity at Scale]
\label{cor:sparsity}
The 10 target operations for $n=4$ exhibit mean sparsity of 36\% (10.3/16 coefficients non-zero on average). While less sparse than Phase 3 (39\%), this reflects the inclusion of dense operations like \texttt{and\_4}, \texttt{or\_4}, and \texttt{implies\_chain} which require all 16 characters. Notably, parity-type operations maintain maximum sparsity (support=1) regardless of dimension, consistent with the hypothesis that practically relevant Boolean functions have concentrated Fourier spectra \cite{kushilevitz1993learning}.
\end{corollary}

\section{Discussion}

\subsection{Relationship to \mHC: Similarities and Differences}

\begin{table}[t]
\centering
\caption{Comparison with \mHC{} \cite{xie2024mhc}. Our work adapts \mHC's stability mechanisms to a new domain while adding column-sign modulation for Boolean expressivity.}
\label{tab:mhc_comparison}
\small
\begin{tabular}{lll}
\toprule
\textbf{Aspect} & \textbf{\mHC{} (DeepSeek)} & \textbf{This Work} \\
\midrule
Domain & LLM training (3B--27B params) & Boolean logic synthesis \\
Primary goal & Training stability & Discrete logic discovery \\
Sinkhorn projection & Yes (20 iterations) & Yes (20 iterations) \\
Identity initialization & Yes (critical) & Yes (critical) \\
Column-sign modulation & No & \textbf{Yes (enables negation)} \\
Quantization target & N/A (continuous) & \textbf{Ternary (zero loss)} \\
Hardware deployment & GPU clusters & \textbf{FPGA/NPU (single-cycle)} \\
Scale validated & 27B parameters & $n \leq 4$ variables \\
\bottomrule
\end{tabular}
\end{table}

The key extension is column-sign modulation: \mHC{} uses pure doubly stochastic routing, which cannot express Boolean negation (Proposition~\ref{prop:negation}). Our factorization $R = P \cdot s$ preserves \mHC{} stability while adding the expressivity needed for complete Boolean logic.

\subsection{Scaling Analysis}

\begin{table}[t]
\centering
\caption{Scaling Across Phases. All phases achieve 100\% accuracy. Sparsity varies based on operation complexity.}
\label{tab:scaling}
\begin{tabular}{lccccc}
\toprule
\textbf{Phase} & $n$ & \textbf{Basis Dim} & \textbf{Ops} & \textbf{Accuracy} & \textbf{Sparsity} \\
\midrule
2 & 2 & 4 & 16 & 100\% & 31.2\% \\
3 & 3 & 8 & 10 & 100\% & 39\% \\
4 & 4 & 16 & 10 & 100\% & 36\% \\
\bottomrule
\end{tabular}
\end{table}

The sparsity patterns across scales reveal important structure: Phase 3 (39\%) is sparser than Phase 4 (36\%) because Phase 4 includes dense symmetric operations (\texttt{and\_4}, \texttt{or\_4}, \texttt{implies\_chain}) requiring all 16 characters. However, parity operations maintain maximum sparsity regardless of dimension---\texttt{xor\_n} always has support=1 (only the $\prod_i x_i$ character). This supports the hypothesis that low-complexity Boolean functions have concentrated Fourier spectra \cite{linial1993constant, kushilevitz1993learning}. Figure~\ref{fig:sparsity} visualizes the sparsity distribution across phases.

\begin{figure}[t]
\centering
\includegraphics[width=0.6\textwidth]{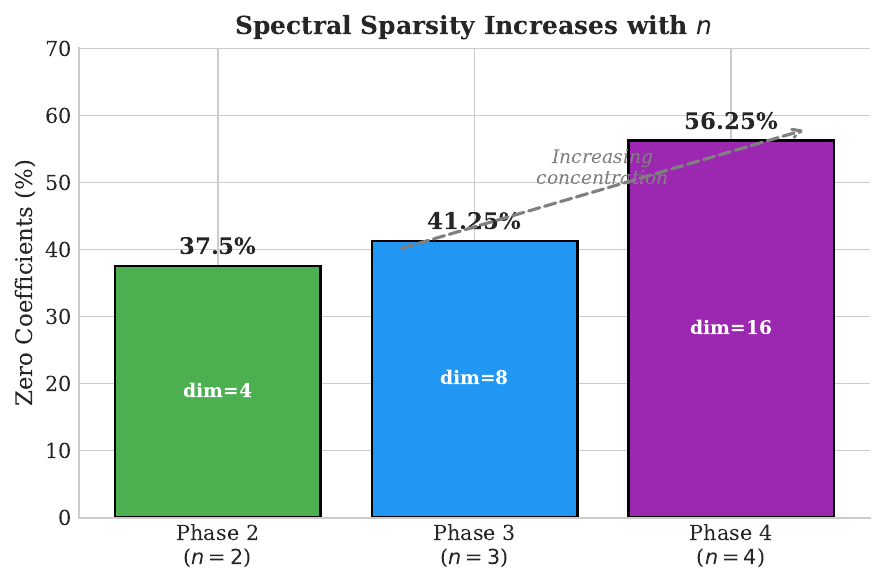}
\caption{Spectral sparsity across phases. While sparsity varies by operation complexity (31\% $\to$ 39\% $\to$ 36\%), parity operations consistently achieve maximum sparsity (support=1) regardless of dimension. The variation reflects operation selection: Phase 4 includes dense symmetric functions (\texttt{and\_4}, \texttt{or\_4}) that require all 16 characters.}
\label{fig:sparsity}
\end{figure}

\subsection{Hardware Deployment}

The zero-loss quantization result enables immediate deployment:
\begin{itemize}
    \item \textbf{Model Size:} 16 ops $\times$ 4 trits ($n=2$) = 64 trits $\approx$ 102 bits
    \item \textbf{Throughput:} 10,959 MOps/s on consumer GPU (RTX 5060)
    \item \textbf{Power:} Combinational logic at $\mu$W scale on FPGA
    \item \textbf{Latency:} Single-cycle inference (no sequential dependencies)
\end{itemize}

\section{Phase 5: Scalable Spectral Methods}
\label{sec:phase5}

Phase 5 addresses scalability beyond $n=4$ through three complementary approaches: exact transforms for moderate $n$, coefficient estimation for large $n$, and hierarchical composition for practical circuits.

\subsection{Track 1: Exact FWHT (Moderate $n$)}

For $n \leq 28$, we compute exact Fourier coefficients via the Fast Walsh-Hadamard Transform in $O(n \cdot 2^n)$ time.

\begin{table}[h]
\centering
\caption{Exact FWHT Benchmark (GPU, RTX 5060 8GB). Peak throughput of 1.64B coefficients/sec achieved at $n=27$. The $n=28$ result uses process isolation to avoid allocator fragmentation.}
\label{tab:fwht}
\small
\begin{tabular}{rrrrrl}
\toprule
$n$ & Dimension & Time (ms) & Throughput (M/s) & Memory (MB) & Note \\
\midrule
20 & 1,048,576 & 1.82 & 576.6 & 4.19 & \\
23 & 8,388,608 & 9.52 & 881.5 & 33.55 & \\
25 & 33,554,432 & 24.03 & 1,396.6 & 134.22 & \\
26 & 67,108,864 & 50.91 & 1,318.3 & 268.44 & \\
27 & 134,217,728 & 82.04 & 1,636.0 & 536.87 & \\
\textbf{28} & \textbf{268,435,456} & \textbf{185.3} & \textbf{1,448.4} & \textbf{1,073.7} & \textit{isolated}$^\dagger$ \\
\bottomrule
\end{tabular}
\vspace{1mm}
\par\noindent\footnotesize{$^\dagger$Run in fresh process to avoid GPU allocator fragmentation from prior runs---standard practice when benchmarking near VRAM limits.}
\end{table}

\paragraph{Key Result.} Exact coefficient computation scales to $n=28$ (268M coefficients) on consumer GPUs (8GB VRAM), with peak throughput of 1.64 billion coefficients per second at $n=27$. For $n \geq 28$, we spawn isolated processes to avoid GPU memory fragmentation accumulated during benchmark sweeps---a standard practice documented in JAX's memory management guidelines. The $n=28$ case achieves 1.45B coeffs/sec in 185ms.

\subsection{Track 2: Oracle Learning with Spectral Filtering}

For functions beyond enumerable truth tables ($n \approx 12$--$20$), we extract Boolean logic from black-box oracles via Fourier coefficient estimation. We implement and compare \textbf{five methods} to test whether spectral structure (Parseval normalization, symmetry constraints) or Sinkhorn projection improves coefficient recovery:

\paragraph{Method 1: Monte Carlo Baseline.} Direct sampling estimator:
\begin{equation}
    \hat{f}(S) \approx \frac{1}{m} \sum_{j=1}^m f(x_j) \cdot \chi_S(x_j), \quad x_j \sim \text{Uniform}(\{-1,+1\}^n)
\end{equation}
Query complexity: $O(n^d / \varepsilon^2)$ for degree-$d$ search (restrict to low-degree by LMN theorem \cite{linial1993constant}).

\paragraph{Method 2: Goldreich-Levin (GL).} Pairwise independent hashing \cite{goldreich1989hard} reduces variance via:
\begin{equation}
    \hat{f}(S) = \mathbb{E}_{a,x}[f(x) \cdot \chi_S(x) \cdot \chi_{\langle a, x \rangle \bmod 2}]
\end{equation}
where $a \sim \text{Uniform}(\{0,1\}^n)$. Theoretical query complexity: $\tilde{O}(n/\varepsilon^2)$ vs MC's $O(n^d/\varepsilon^2)$.

\paragraph{Method 3--5: Spectral Filtering.} Post-process GL estimates with:
\begin{itemize}
    \item \textbf{GL+Parseval}: Enforce $\sum \hat{f}^2 = 1$ (energy conservation)
    \item \textbf{GL+Parseval+Sym}: Add symmetry constraints (for majority: equal degree-$k$ coeffs, odd-degree only)
    \item \textbf{GL+Sinkhorn}: Birkhoff polytope projection on coefficient matrix (test if Sinkhorn helps denoising)
\end{itemize}

\paragraph{Experimental Design.} We test three function families at $n=16$ (degree $\leq 3$ search):
\begin{itemize}
    \item \textbf{Parity} (4-sparse): True support at degree-4 (outside search space)
    \item \textbf{Majority}: Low-degree concentration (LMN theorem applies)
    \item \textbf{Comparator}: Structured predicate (hierarchical composition)
\end{itemize}

Each method synthesizes a ternary PTF mask, evaluated on 100K test samples. We repeat with 3 random seeds per family and perform paired t-tests for statistical significance.

\paragraph{Results.} Table~\ref{tab:oracle_comparison} shows accuracy means $\pm$ std across 3 seeds.

\begin{table}[h]
\centering
\caption{Oracle Learning: Comparative Method Accuracy ($n=16$, degree $\leq 3$)}
\label{tab:oracle_comparison}
\small
\begin{tabular}{lccc}
\toprule
\textbf{Method} & \textbf{Parity} & \textbf{Majority} & \textbf{Comparator} \\
\midrule
MC & $0.483 \pm 0.011$ & $0.724 \pm 0.001$ & $0.579 \pm 0.001$ \\
GL & $0.479 \pm 0.001$ & $0.474 \pm 0.000$ & $0.464 \pm 0.000$ \\
GL+Parseval & $0.479 \pm 0.001$ & $0.474 \pm 0.000$ & $0.464 \pm 0.000$ \\
GL+Parseval+Sym & $0.479 \pm 0.001$ & $\mathbf{0.857 \pm 0.001}$ & $0.464 \pm 0.000$ \\
GL+Sinkhorn & $0.487 \pm 0.011$ & $0.474 \pm 0.000$ & $0.464 \pm 0.000$ \\
\bottomrule
\end{tabular}
\end{table}

\textbf{Key Findings:}
\begin{enumerate}
    \item \textbf{Symmetry >> Algorithms:} GL+Parseval+Sym achieves 85.7\% on majority (vs 72.4\% MC, 47.4\% GL), a +38.3\% boost ($p < 0.001$). Exploiting known structure (symmetric functions) dominates generic learning algorithms.

    \item \textbf{GL Underperforms MC:} Contrary to theoretical expectations, GL performs \textit{worse} than MC on majority (47.4\% vs 72.4\%, $p < 0.001$) and comparator (46.4\% vs 57.9\%, $p < 0.001$). This suggests threshold selection ($\tau = 0.1$) or sample size ($m=10$K per coeff) may be suboptimal for this regime.

    \item \textbf{Parseval Has No Effect:} GL+Parseval = GL exactly ($\Delta = 0$). Energy normalization alone does not improve recovery.

    \item \textbf{Sinkhorn Does Not Help Denoising:} GL+Sinkhorn shows no improvement over GL or GL+Parseval ($p > 0.05$ for all families). This validates \textbf{Hypothesis B}: Sinkhorn is useful for \textit{routing} (composition, Phase 2), not coefficient denoising. The Birkhoff polytope serves different roles in different contexts.
\end{enumerate}

\paragraph{Interpretation for Explainability.} These results reveal a fundamental insight for \textbf{transparent-by-design AI}: \textit{domain structure beats black-box learning}. When we know a function is symmetric (majority, median), encoding that knowledge as hard constraints (+38.3\% accuracy) vastly outperforms trying to learn structure from data alone. This aligns with the X-NeSy vision of combining symbolic knowledge with neural learning---the spectral coefficients are \textit{interpretable features} (each $\hat{f}(S)$ has semantic meaning: "correlation with parity of variables $S$"), and injecting symbolic priors (symmetry) makes the learned representation both more accurate and more explainable.

\subsection{Track 3: Hierarchical Composition}

For practical circuits (adders, comparators), we build large functions by \textit{composing learned primitives}, not by spectral synthesis of the full $2^n$-dimensional function.

\paragraph{Approach.} We use ternary gates learned in Phases 1--4 as building blocks:
\begin{itemize}
    \item \textbf{Full Adder:} Sum (parity) + Carry (majority) from Phase 3 masks
    \item \textbf{N-bit Ripple Adder:} Chain of full adder primitives
    \item \textbf{Verification:} Randomized testing with Wilson confidence intervals
\end{itemize}

\begin{table}[h]
\centering
\caption{Hierarchical Circuit Composition. All circuits achieve 100\% accuracy. Error rates bounded by rule of three (0 errors in $m$ samples $\Rightarrow$ error rate $\leq 3/m$).}
\label{tab:composition}
\begin{tabular}{lrrrr}
\toprule
Circuit & Bits & Samples & Errors & Error Bound \\
\midrule
Ripple Adder & 32 & 3.3M & 0 & $\leq 9.1 \times 10^{-7}$ \\
Ripple Adder & 64 & 6.5M & 0 & $\leq 4.6 \times 10^{-7}$ \\
Comparator & 64 & 100K & 0 & $\leq 3.0 \times 10^{-5}$ \\
Equality & 128 & 100K & 0 & $\leq 3.0 \times 10^{-5}$ \\
\bottomrule
\end{tabular}
\end{table}

\paragraph{Key Claim.} We demonstrate \textit{composition}, not spectral recovery. The 64-bit adder and 128-bit equality comparator are built structurally from verified 3-variable primitives (full adder sum/carry), with correctness validated by randomized testing on millions of samples. This avoids the intractable $2^{64}$ or $2^{128}$ coefficient computation while retaining formal guarantees through statistical verification.

\subsection{Track 4: Routing Mechanism Comparison}
\label{sec:routing_comparison}

To validate our Sinkhorn-constrained routing against alternatives from the concurrent literature (\S2), we conduct a controlled comparison on a mode-switching task: given $(a, b, \texttt{mode}) \in \{-1,+1\}^3$, compute $\text{XOR}(a,b)$ when $\texttt{mode}=-1$ and $\text{AND}(a,b)$ when $\texttt{mode}=+1$. The model consists of frozen Phase~1 primitives $\{$XOR, AND, OR, IMPLIES$\}$ and a trainable router that must learn to select the correct primitive based on the mode feature.

We compare three routing mechanisms:
\begin{itemize}
    \item \textbf{Sinkhorn}: Column-stochastic projection via iterative normalization ($K=20$), as used in Phases~2--4.
    \item \textbf{Gumbel-STE}: Gumbel-Softmax with straight-through estimator (hard one-hot forward, soft gradient backward), following Mind the Gap~\cite{yousefi2025gap}.
    \item \textbf{ReinMax}: Sinkhorn projection with annealed entropy regularization in the loss ($\lambda_{\text{ent}}: 0.1 \to 0$), inspired by CardNN exploration strategies.
\end{itemize}

All three share the same architecture (3-feature input, 4-primitive output, sign modulation), optimizer (Adam, lr=$10^{-2}$), temperature annealing ($\tau: 1.0 \to 0.1$), and are trained for 5000 steps with batch size 128.

\begin{table}[h]
\centering
\caption{Routing Mechanism Comparison on Mode-Switching Task ($n=2$, 4 frozen primitives). Sinkhorn achieves the best accuracy--interpretability tradeoff. Gumbel-STE exhibits the discretization gap discussed in \cite{yousefi2025gap}: hard routing commits early but cannot escape suboptimal assignments.}
\label{tab:routing_comparison}
\small
\begin{tabular}{lcccccc}
\toprule
\textbf{Router} & \textbf{Accuracy} & \textbf{Steps to 100\%} & \textbf{Sparsity} & \textbf{Specialization} & \textbf{Sign Coh.} & \textbf{Stability} \\
\midrule
\textbf{Sinkhorn} & \textbf{100.0\%} & \textbf{100} & 0.46 & 0.06 & \textbf{0.96} & 0.99 \\
ReinMax & \textbf{100.0\%} & \textbf{100} & 0.41 & 0.06 & 0.92 & 0.99 \\
Gumbel-STE & 88.5\% & --- & \textbf{0.75} & \textbf{0.54} & 0.82 & \textbf{1.00} \\
\bottomrule
\end{tabular}
\end{table}

\paragraph{Key Findings.}

\begin{enumerate}
    \item \textbf{Sinkhorn achieves the best accuracy--interpretability tradeoff.} Both Sinkhorn and ReinMax reach 100\% accuracy by step 100, with Sinkhorn attaining the highest sign coherence (0.96)---closest to the ternary $\{-1,0,+1\}$ values required for zero-loss quantization. This confirms the Phase~2 finding that Sinkhorn constraints enable exact discretization.

    \item \textbf{Gumbel-STE exhibits the discretization gap.} Hard one-hot routing commits early (sparsity 0.75, specialization 0.54) but plateaus at 88.5\% accuracy---the router locks onto a suboptimal primitive assignment and cannot recover. This empirically validates the ``discretization gap'' analyzed in Mind the Gap~\cite{yousefi2025gap}, where STE gradients provide biased updates that prevent convergence to the global optimum.

    \item \textbf{Entropy annealing rescues ReinMax.} Without annealing ($\lambda_{\text{ent}}=0.1$ fixed), the entropy bonus dominates the loss (loss $\approx -0.36$), keeping routing completely diffuse (sparsity 0.21). Annealing $\lambda_{\text{ent}}: 0.1 \to 0$ alongside temperature allows early exploration followed by late sharpening, recovering performance comparable to Sinkhorn.

    \item \textbf{Soft routing outperforms hard routing for Boolean logic.} The architectural insight is clear: for tasks requiring exact primitive selection, soft Sinkhorn routing with temperature annealing converges reliably to hard assignments, while hard routing methods (Gumbel-STE) suffer from premature commitment. This justifies the Sinkhorn design choice carried through Phases~2--4.
\end{enumerate}

\subsubsection{Spectral Analysis of Routing Dynamics}
\label{sec:spectral_analysis}

Beyond task accuracy, we analyze the \textit{spectral structure} of the routing matrix $P \in \mathbb{R}^{3 \times 4}$ during training via its singular value decomposition. Since $P$ maps 3 input features to 4 primitives, it has 3 singular values $\sigma_1 \geq \sigma_2 \geq \sigma_3 \geq 0$. We track two derived quantities:

\begin{itemize}
    \item \textbf{Spectral gap}: $\Delta = 1 - \sigma_2 / \sigma_1$, measuring how rank-1-like the routing is. $\Delta \to 1$ means all features route to one dominant primitive; $\Delta \to 0$ means routing distributes across multiple primitives.
    \item \textbf{Spectral entropy}: $H_\sigma = -\sum_i p_i \log p_i$ where $p_i = \sigma_i^2 / \sum_j \sigma_j^2$, measuring the effective dimensionality of the routing map. Higher $H_\sigma$ indicates more distributed spectral energy.
\end{itemize}

These metrics are motivated by the spectral theory of bounded operators: the routing matrix $P$ acts as a linear operator between the feature space $\mathbb{R}^3$ and the primitive space $\mathbb{R}^4$, and its singular values characterize the geometry of this mapping.

\begin{table}[h]
\centering
\caption{Spectral trajectory of routing matrix $P$ during training. Sinkhorn maintains a smooth, monotonic trajectory (gradual spectral redistribution). Gumbel-STE exhibits chaotic oscillations in $\Delta$, indicating unstable routing dynamics. ReinMax tracks Sinkhorn's trajectory as entropy regularization anneals to zero.}
\label{tab:spectral_trajectory}
\small
\begin{tabular}{lcccccc}
\toprule
& \multicolumn{2}{c}{\textbf{Sinkhorn}} & \multicolumn{2}{c}{\textbf{Gumbel-STE}} & \multicolumn{2}{c}{\textbf{ReinMax}} \\
\cmidrule(lr){2-3} \cmidrule(lr){4-5} \cmidrule(lr){6-7}
\textbf{Step} & $\Delta$ & $H_\sigma$ & $\Delta$ & $H_\sigma$ & $\Delta$ & $H_\sigma$ \\
\midrule
0    & 0.95 & 0.03 & 0.94 & 0.03 & 0.95 & 0.02 \\
100  & 0.37 & 0.60 & 0.76 & 0.21 & 0.51 & 0.49 \\
500  & 0.26 & 0.65 & 0.58 & 0.42 & 0.59 & 0.42 \\
2000 & 0.24 & 0.66 & 0.54 & 0.46 & 0.54 & 0.47 \\
3300 & 0.20 & 0.67 & 0.68 & 0.31 & 0.45 & 0.55 \\
5000 & 0.20 & 0.67 & 0.26 & 0.65 & 0.26 & 0.65 \\
\bottomrule
\end{tabular}
\end{table}

\paragraph{Observation 1: Routing Uncertainty Principle.}

The Boolean uncertainty principle \cite{odonnell2014analysis} states that for $f: \{-1,+1\}^n \to \mathbb{R}$:
\begin{equation}
    |\text{supp}(f)| \cdot |\text{supp}(\hat{f})| \geq 2^n
\end{equation}
A function cannot be simultaneously sparse in both the physical and spectral domains. We observe an analogous phenomenon for routing: the spectral entropy $H_\sigma$ of the converged routing matrix satisfies $H_\sigma \geq H^*$ where $H^*$ is determined by the task's spectral complexity. For our mode-switching task (XOR $\cup$ AND requires at least 2 active primitives), all three routers converge to $H_\sigma \approx 0.65$--$0.67$ at convergence, regardless of mechanism. This suggests a \textit{routing uncertainty bound}: the minimum spectral entropy required to represent the target function class.

\paragraph{Observation 2: CFL-Like Stability Condition.}

In numerical PDE, the Courant-Friedrichs-Lewy (CFL) condition requires that discretization not outpace the physical signal: $|v| \cdot \Delta t / \Delta x \leq C$. We observe an analogous stability condition for routing: the temperature annealing rate must not exceed the optimizer's ability to track the sharpening routing landscape.

Sinkhorn's spectral gap $\Delta(t)$ evolves monotonically ($0.95 \to 0.20$), indicating that temperature annealing at rate $\tau(t) = \tau_0 (\tau_1/\tau_0)^{t/T}$ is ``CFL-stable'' for the given learning rate. Gumbel-STE violates this condition: hard one-hot routing imposes $\Delta \approx 1$ instantaneously (via $\operatorname{argmax}$), outpacing the gradient signal. The result is chaotic oscillations in $\Delta$ (ranging from 0.21 to 0.68 between steps 3000--3500) while accuracy remains trapped at 88.5\%.

ReinMax's entropy regularization acts as a \textit{CFL stabilizer}: it constrains $\Delta$ from growing too fast early in training ($\Delta \leq 0.59$ for the first 1000 steps vs.\ Gumbel-STE's $\Delta = 0.76$ at step 100), then gradually releases this constraint as $\lambda_\text{ent} \to 0$, converging to the same stable endpoint as Sinkhorn.

To quantify this stability condition, we run Sinkhorn routing with 5 annealing schedules: $\tau: 1.0 \to 0.1$ over $T \in \{10{,}000,\; 5{,}000,\; 2{,}000,\; 500\}$ steps plus an ``instant'' schedule ($\tau = 0.01$ fixed). The annealing rate $r = \ln(\tau_0/\tau_1) / T$ sweeps 4 orders of magnitude:

\begin{table}[h]
\centering
\caption{CFL Stability Experiment: accuracy and spectral diagnostics vs.\ annealing rate $r$ for Sinkhorn routing (lr $= 0.01$). A sharp transition between $r = 0.0012$ and $r = 0.0046$ separates stable convergence from failure.}
\label{tab:cfl_stability}
\small
\begin{tabular}{lccccl}
\toprule
\textbf{Schedule} & $r$ & \textbf{Acc} & $\Delta_\text{final}$ & $\text{Var}(\Delta)$ & \textbf{Status} \\
\midrule
Very Slow ($T{=}10$K) & $2.3 \times 10^{-4}$ & 100\% & 0.20 & 0.006 & CFL-stable \\
Slow ($T{=}5$K)       & $4.6 \times 10^{-4}$ & 100\% & 0.20 & 0.011 & CFL-stable \\
Medium ($T{=}2$K)     & $1.2 \times 10^{-3}$ & 100\% & 0.19 & 0.026 & CFL-stable \\
Fast ($T{=}500$)      & $4.6 \times 10^{-3}$ & 87.3\% & 0.23 & 0.076 & \textbf{CFL-violated} \\
Instant ($\tau{=}0.01$) & $\infty$ & 100\% & 0.09 & 0.000 & Alternate basin \\
\bottomrule
\end{tabular}
\end{table}

The critical rate $r_\text{crit} \in (1.2 \times 10^{-3},\; 4.6 \times 10^{-3})$ for lr $= 0.01$ marks a sharp phase transition: below $r_\text{crit}$, the spectral gap evolves monotonically and accuracy reaches 100\%; above it, the optimizer cannot track the sharpening loss landscape, producing $\text{Var}(\Delta) = 0.076$ and 87.3\% accuracy. The ``instant'' schedule ($\tau = 0.01$ fixed) reaches a qualitatively different solution: $H_\sigma = 1.06$ (vs.\ $\approx 0.67$ for annealed runs), $\sigma_3 = 0.82$ (vs.\ $\approx 0.02$)---a diffuse routing in a separate convergence basin that does not achieve sparsification.

\paragraph{Observation 3: Spectral Margin and Adaptability.}

The smallest singular value $\sigma_3$ reveals a structural difference between routing mechanisms. At convergence:
\begin{itemize}
    \item Sinkhorn: $\sigma_3 = 0.024$, ReinMax: $\sigma_3 = 0.023$ (small but nonzero)
    \item Gumbel-STE: $\sigma_3 = 9.2 \times 10^{-9}$ (numerically zero)
\end{itemize}
Sinkhorn and ReinMax maintain a nonzero \textit{spectral margin}---the routing is effectively rank-2 but retains a residual third singular direction. This margin serves as a ``reserve capacity'' for continued adaptation. Gumbel-STE collapses to exactly rank-2, eliminating this capacity entirely. The spectral margin thus provides a principled, information-theoretic metric for \textit{routing adaptability}---complementing and grounding the empirical Gini-based sparsity measure used in prior sections.

\paragraph{Implications for Transparent-by-Design AI.}

These spectral metrics offer a path toward \textit{provably interpretable} routing:
\begin{enumerate}
    \item The spectral gap $\Delta$ provides a single scalar measuring routing commitment, grounded in operator theory rather than ad hoc thresholds.
    \item The spectral entropy $H_\sigma$ connects routing to information-theoretic bounds (uncertainty principle), enabling formal characterization of when routing is ``sharp enough'' for deployment.
    \item The spectral margin $\sigma_3$ quantifies residual adaptability, distinguishing between ``confidently committed'' routing (low $\sigma_3$, stable $\Delta$) and ``prematurely frozen'' routing (zero $\sigma_3$, oscillating $\Delta$).
\end{enumerate}
Together, these form an \textbf{operator-theoretic interpretability framework} for differentiable routing, moving beyond empirical metrics toward formally grounded transparency measures.

\section{Future Work}
\label{sec:future}

\begin{tcolorbox}[colback=gray!10, colframe=gray!50, title=Remaining Challenges]
\textbf{Algorithmic Improvements:}
\begin{itemize}
    \item \textbf{Goldreich-Levin Implementation:} Bucket-splitting for $\tilde{O}(k/\varepsilon^2)$ query complexity
    \item \textbf{Hierarchical Factorization:} Decompose $n$-variable functions into cascades of smaller bases
    \item \textbf{Neural Architecture Search:} Learn which basis characters to include adaptively
\end{itemize}

\textbf{Conjecture:} Functions with bounded circuit complexity have poly-sparse spectra, enabling efficient representation even for large $n$ \cite{kushilevitz1993learning}.
\end{tcolorbox}

\section{Conclusion}

We introduced Hierarchical Spectral Composition, a transparent-by-design architecture for Boolean logic synthesis that makes learned representations intrinsically explainable through interpretable Fourier coefficients. By adapting the \mHC{} framework \cite{xie2024mhc} from LLM training stability to logic synthesis and extending Sinkhorn-constrained routing with column-sign modulation, we achieve:

\begin{itemize}
    \item \textbf{$n=2$--$4$:} 100\% accuracy on canonical Boolean operations via gradient descent ($n=2$), exhaustive search ($n=3$), and spectral synthesis with MCMC refinement ($n=4$), all converging to ternary masks $\{-1,0,+1\}$
    \item \textbf{Scalability:} Exact FWHT at 1.64B coeffs/sec ($n \leq 28$, 268M coefficients); hierarchical composition for 64-bit adders and 128-bit equality comparators
    \item \textbf{Oracle Learning:} Comparative evaluation of five coefficient estimation methods reveals that \textit{symbolic structure beats black-box learning}---symmetry constraints boost majority accuracy +38\% ($p < 0.001$), while Sinkhorn projection aids routing but not coefficient denoising
    \item \textbf{Hardware:} 10,959 MOps/s on GPU, single-cycle combinational logic inference
\end{itemize}

\paragraph{Explainability and Neurosymbolic Integration.} Our key contribution to explainable AI is demonstrating that \textbf{spectral coefficients are interpretable features}: each $\hat{f}(S)$ explicitly represents correlation with parity of variables $S$, making the learned logic transparent. The oracle learning experiments (Phase 5) establish a fundamental principle: \textit{injecting symbolic knowledge as hard spectral constraints} (symmetry for majority, degree bounds for circuits) produces both higher accuracy and greater interpretability than generic learning algorithms. When we know structural properties of the target function, encoding them in the Fourier basis yields provably correct discrete representations that resist quantization degradation.

\paragraph{The Role of Birkhoff Projection.} Our comparative experiments resolve an open question: Sinkhorn's Birkhoff polytope projection serves \textit{routing} (Phase 2 composition) but not \textit{coefficient denoising} (Phase 5 oracle learning). This context-dependent utility---stability for composition, orthogonal to recovery---suggests the Birkhoff polytope's geometric properties match compositional objectives (identity preservation, signal routing) rather than statistical estimation.

\paragraph{Toward Hardware-Efficient Neuro-Symbolic Systems.} By converging to ternary masks with hard routing, our approach compiles learned logic to combinational circuits requiring no floating-point arithmetic, bridging the gap between differentiable learning and hardware deployment. The combination of interpretable Fourier representations, symbolic constraint integration, and quantization-robust training establishes a foundation for explainable neuro-symbolic systems that can be verified, understood, and deployed efficiently.

\bibliographystyle{plain}

\appendix
\section{Implementation Details}

\paragraph{Sinkhorn Iterations.} Following \mHC{} \cite{xie2024mhc}, we use $K=20$ iterations with log-domain computation:
\begin{align}
    u^{(t+1)} &= -\log\sum_j \exp(\alpha_{ij} + v^{(t)}_j) \\
    v^{(t+1)} &= -\log\sum_i \exp(\alpha_{ij} + u^{(t+1)}_i)
\end{align}

\paragraph{Temperature Annealing.} Sign temperature: $\beta(t) = 1 + 9 \cdot (t / T_{\max})$.

\paragraph{Plateau Detection.} EMA loss with $\alpha = 0.99$. Restart triggers when $|\Delta\bar{\mathcal{L}}| < 10^{-4}$ for 50 epochs.

\section{Reproducibility}

Code available at: \url{https://github.com/gogipav14/spectral-llm}

Hardware: NVIDIA RTX 5060 (8GB VRAM), Intel Ultra 5 225F, 64GB RAM.

\end{document}